\documentclass[sigplan,screen]{acmart}
\AtBeginDocument{%
  \providecommand\BibTeX{{%
    \normalfont B\kern-0.5em{\scshape i\kern-0.25em b}\kern-0.8em\TeX}}}
\renewcommand\footnotetextcopyrightpermission[1]{}
\usepackage{multirow} 

\usepackage[linesnumbered,lined,ruled]{algorithm2e}
 
\usepackage{arydshln} 
\usepackage{graphicx}
\usepackage{subcaption}
\usepackage{paralist}

\newtheorem{assumption}{Assumption}
\newtheorem{theorem}{Theorem}
\newtheorem{lemma}{Lemma}

\begin{document}

\author{Dengke Yan}
\affiliation{%
  \institution{East China Normal University}
  \streetaddress{200062}
  \city{Shanghai}
  \country{China}
}

\author{Ming Hu}
\affiliation{%
  \institution{Nanyang Technological University}
  \city{Singapore}
  \country{Singapore}
}
\authornote{Corresponding authors}
\email{hu.ming.work@gmail.com}

\author{Zeke Xia}
\affiliation{%
  \institution{East China Normal University}
  \streetaddress{200062}
  \city{Shanghai}
  \country{China}
}

\author{Yanxin Yang}
\affiliation{%
  \institution{East China Normal University}
  \streetaddress{200062}
  \city{Shanghai}
  \country{China}
}

\author{Jun Xia}
\affiliation{%
  \institution{University of Notre Dame}
  \city{Notre Dame}
  \country{USA}
}

\author{Xiaofei Xie}
\affiliation{%
  \institution{Singapore Management University}
  \city{Singapore}
  \country{Singapore}
}

\author{Mingsong Chen}
\affiliation{%
  \institution{East China Normal University}
  \streetaddress{200062}
  \city{Shanghai}
  \country{China}
}
\authornotemark[1]
\email{mschen@sei.ecnu.edu.cn}

\title{
Have Your Cake and Eat It Too: Toward Efficient and Accurate Split Federated Learning 
}

\begin{abstract}
In resource-constrained AIoT systems, traditional Federated Learning (FL) approaches cannot deploy complete models on edge devices. Split Federated Learning (SFL) solves this problem by splitting the model into two parts and training one of them on edge devices.
However, due to data heterogeneity and stragglers, SFL suffers from the challenges of low inference accuracy and low efficiency.
To address these issues, this paper presents a novel SFL approach, named Sliding Split Federated Learning (S$^2$FL), which adopts an adaptive sliding model split strategy and a data balance-based training mechanism.
Specifically, S$^2$FL dynamically dispatches different model portions to AIoT devices according to their computing capability so that devices with different computational resources have approximate training times, thus alleviating the low training efficiency caused by stragglers.
By combining features uploaded by devices with different data distributions to generate multiple larger batches with a uniform distribution for back-propagation, S$^2$FL can alleviate the performance degradation caused by data heterogeneity.
Experimental results demonstrate that, compared to conventional SFL, S$^2$FL can achieve up to 16.5\% inference accuracy improvement and $3.54\times$ training acceleration.

\end{abstract}


\maketitle

\section{Introduction}
With the improvement of edge device computing capabilities, deep learning-based applications are popular in Artificial Intelligence of Things (AIoT) systems~\cite{chang2021survey,zhang2020empowering,li2021privacy,hu2023aiotml} (e.g., Traffic Management, Automated Driving, Smart Home, and Industrial Manufacturing). These AIoT devices process data received by sensors and make real-time decisions through deployed deep learning models.
As a promising distributed deep learning paradigm, Federated Learning (FL)~\cite{fedavg,li2020federated,hu2024fedmut,cui2022optimizing,li2023towards} has been widely used in AIoT applications~\cite{li2019smartpc,li2020deepfed,zhang2020efficient, guo2021lightfed, wang2019edge, hu2023gitfl}.
By dispatching a global model to massive AIoT devices for model training and aggregating the trained local models to generate a new global model, FL enables collaborative model training without directly obtaining the raw data of each device.
However, compared to the cloud server, the hardware resources on AIoT devices are still seriously limited, resulting in the fact that conventional FL cannot support large-size model training on local devices.

To address the above issue, Split Federated Learning (SFL)~\cite{thapa2022splitfed} has been presented for model training in resource-constrained scenarios, which splits part of the model training process performed on local devices to the cloud server. 
In specific, in SFL, the global model is split into two portions, i.e., the client-side model portion and the server-side model portion, where the client-side model portion is trained on local devices and the server-side model portion is trained on the cloud server. 
To prevent the risk of privacy leakage, SFL includes two servers, i.e., the {\it {\it Fed Server}} and the {\it {\it Main Server}}.
In each SFL training round, the {\it Fed Server} dispatches the client-side model portion to activated devices. 
Each device inputs its raw data to the client-side model portion and performs forward operation to obtain intermediate features and uploads the features with their corresponding data labels to the {\it Main Server}.
The {\it Main Server} generates multiple copies of the server-side model portion and uses received features to train these copies and performs backward operation to obtain the gradient of these features, where a copy is trained by features from a device, and the gradient will be dispatched to the corresponding device.
By aggregating all the trained copies, the {\it Main Server} can generate a new server-side model portion.
The devices use the received gradient to update their client-side model portion and upload the trained portion to the {\it Fed Server}. By aggregating trained client-side model portions, the {\it Fed Server} can generate a new client-side model portion.


Although SFL enables large-size model training among resource-constrained AIoT devices, it still encounters the problem of low inference accuracy due to data heterogeneity and the problem of low training efficiency caused by device heterogeneity.
Typically, since massive AIoT devices are located in different physical environments and have different data collection preferences, their local data are not Independent Identically Distributed (non-IID), resulting in the notorious ``client drift'' problem~\cite{karimireddy2020scaffold,gao2022feddc,hu2022fedcross}.
Specifically, non-IID data cause local models to be optimized towards different directions, which results in the low inference accuracy of the aggregated global model.
Although conventional FL uses various strategies to alleviate non-IID problems, e.g., SCAFFOLD~\cite{karimireddy2020scaffold} tries to estimate the update direction for the server model and the update direction for each client to correct for client drift problem, IFCA~\cite{ghosh2020efficient} identify each user's cluster membership and optimize each cluster model in a distributed setup by alternately estimating the user's clustering identity and optimizing the model parameters for the user's clustering, FedGen ~\cite{zhu2021data} is a data-free knowledge distillation approach to address heterogeneous FL, where the server learns a lightweight generator to ensemble user information in a data-free manner, which is then broadcasted to users. Since the main cloud server cannot obtain the full global model, the existing strategies cannot be directly integrated into SFL.
Moreover, due to the heterogeneity of hardware resources (e.g., CPUs and GPUs), the computing capabilities of different AIoT devices vary greatly, which results in large differences in the local training time of different devices.
Since model aggregation needs to collect the models of all activated devices, the efficiency of FL training is severely limited by low-performance devices (i.e., stragglers)~\cite{ma2021fedsa,xie2019asynchronous}. 
\textit{Therefore, how to improve the accuracy and efficiency of SFL in non-IID scenarios has become an important challenge.}

Since in SFL, the {\it Main Server} can obtain the labels of uploaded features, the {\it Main Server} can group the features according to their labels to make the label distribution of each group close to IID.
Intuitively, by using the grouped features with IID labels to train server-side model portion copies, SFL can alleviate model performance degradation caused by data heterogeneity.
Moreover, SFL can dispatch client-side model portions with different sizes to different devices, where the high-performance device is dispatched to a large-size client-side model portion and the low-performance device is dispatched to a small-sized client-side model portion.  
In this way, due to the dispatch of a small client-side model portion, low-performance devices can spend less time on model training, allowing them to keep pace with high-performance devices.

\begin{figure*}[h] 
	\begin{center} 
		\includegraphics[width=0.95\textwidth]{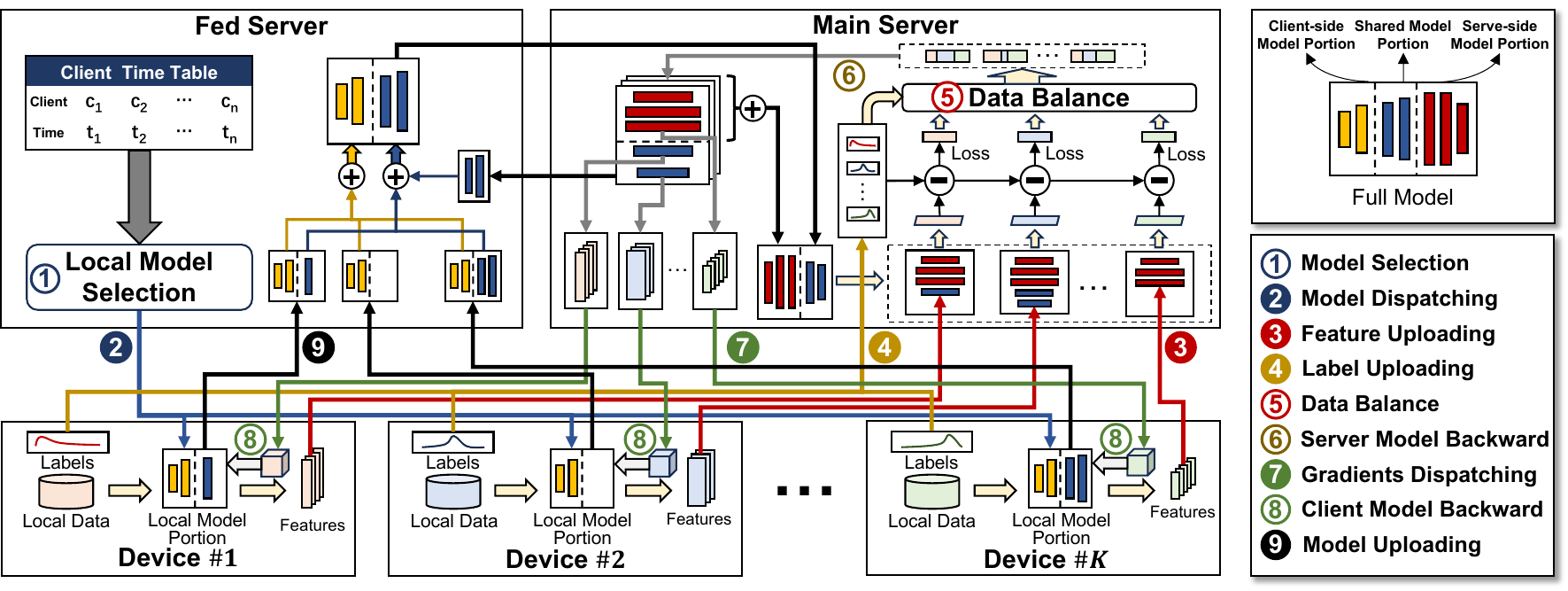}
		\caption{Framework and workflow of $S^2FL$.}
		\label{fig:framework} 
	\end{center}
\end{figure*}

Inspired by the above intuitions, this paper presents a novel SFL approach, named \textbf{S}liding \textbf{S}plit \textbf{F}ederated \textbf{L}earning (S$^2$FL), which adopts a data balance-based training mechanism and an adaptive sliding model split strategy.
In specific, in S$^2$FL, to alleviate data heterogeneity, the {\it Main Server} recombines intermediate features according to their labels to ensure that each server-side model portion is trained by more IID features.
To deal with stragglers, S$^2$FL split the global model into three portions, i.e., the server-side model portion, the shared model portion, and the client-side model portion. At each S$^2$FL training round, the {\it Fed Server} splits the shared model portion according to the hardware resource of each AIoT device and dispatches the split shared model portion together with the client-side model portion to the corresponding device. 
In this way, AIoT devices with varying hardware capabilities can be dispatched different sizes of model portions, thus alleviating the training time differences caused by stragglers.
The main contributions of this paper are as follows:
\begin{itemize}
\item[$\bullet$] We propose S$^2$FL, a novel SFL framework, which adaptively splits models according to the device hardware resource, making the training time similar between devices to alleviate low training efficiency caused by straggers.

\item[$\bullet$] We propose a data balancing-based training mechanism to address the problem of reduced inference accuracy due to data heterogeneity by grouping intermediate features according to their labels and making the training data distribution in the server-side model part more balanced.

\item[$\bullet$] We conduct experiments on various well-known datasets with various data heterogeneity scenarios to evaluate the effectiveness of our S$^2$FL approach.

\end{itemize}

The remaining parts of this paper are organized as follows. Section 2 introduces background and related works on Federated Learning.  Section 3 details our proposed $S^2FL$ approach. Section 4 analyzes the convergence of our approach. Section 5 presents the performance evaluation results of our approach. Finally, Section 6 concludes this paper.

\section{Background and Related Work}

\textbf{Federated Learning} (FL)~\cite{fedavg} is a distributed machine learning framework that addresses data privacy problems. It does not require each client to transmit local data to the cloud server. Still, each client directly trains the model locally and sends the insensitive model parameter information to the cloud server for aggregation in each round. The cloud server synchronizes the aggregated model parameters as the model parameters of the new round to each client. However, in the AIoT scenario, different edge devices possess varying computational capabilities and data distributions, which will lead to straggler problem and non-iid problem in FL. The straggler problem refers to fast devices need to wait for slow devices during training, thus wasting the computational resources of fast devices. This is because the time consumed to train the same model on different clients is not the same as the hardware resources of each client are different. However, the synchronized aggregation strategy used in FL needs to wait for the models of all the clients to be trained before starting the aggregation process. To solve the straggler problem, some works avoid or minimize inter-client waits by using asynchronous or semi-asynchronous aggregation\cite{xie2019asynchronous, wu2020safa, ma2021fedsa, che2022decentralized, chai2021fedat, cao2021hadfl}. For example, FedAT first groups all clients according to their training time and puts clients with similar training time into the same group. The models of clients within the same group are trained synchronously and aggregated into one model, and the models between groups are aggregated asynchronously. And there are other works balance the training time between clients by assigning each client a model that matches its hardware resources\cite{9586241, nishio2019client, jiang2022model}, such as Helios drops some nodes on the CNN model in proportion to the computational power of a single client, so that training on lagging devices can also be synchronized with fast devices trained on the full model. Non-iid problem means the inconsistent distribution of data between clients, which is usually reflected in the amount of data and the distribution of label. The non-IID problem causes the model to favor some clients with more data or to favor some labels with a greater number, which decreases the model's inference accuracy or makes it even unable to converge\cite{zhao2018federated}. Some current work addresses the non-IID problem through device selection\cite{zhang2023fedmds,long2023multi} and knowledge distillation\cite{itahara2021distillation, li2019fedmd, ma2022continual}. Although these works can effectively solve the straggler problem or the non-iid problem, they require the client to deploy a complete model. In the AIoT scenario with a large number of edge devices, devices with limited computational resources may not be able to train the complete model independently. At the same time, transferring the complete model between a large number of devices and servers imposes a significant communication burden, especially for asynchronous aggregation methods that may require frequent communication.

\textbf{Split Federated Learning (SFL)}~\cite{thapa2022splitfed} is a method that combines Federated Learning and Split Learning (SL)~\cite{gupta2018distributed,vepakomma2019reducing}. SL divides the complete model into several parts, with each model part deployed on a separate client for computation. The clients with adjacent model parts transmit the intermediate feature which is the output of their own model's last layer. When the model is large, SL effectively distributes the computing tasks to various clients, allowing clients that cannot independently load and train the complete model to participate in the model training. Furthermore, due to the smaller size of the feature compared to model parameters, SL has lower communication overhead compared to FL. Although SL has advantages when dealing with large models, the computation between clients is sequential, leading to a significant number of clients being idle during the training process. This results in a waste of computational resources and large training time overhead. In SFL, the complete model is divided into two parts: the client-side model portion and the server-side model portion. Each client communicates directly with the {\it Main Server} and {\it Fed Server}. In each training round, clients interact with the {\it Main Server} in parallel to execute the SL process. Subsequently, clients send their updated client model portion to the {\it Fed Server} for aggregation. The {\it Fed Server} aggregates all clients' models and synchronizes the aggregated model with all clients. However, a shortcoming of SFL is that it still experiences the straggler problem and non-iid problem in FL.                                                       

To the best of our knowledge, $S^2FL$ is the first attempt that address the straggler problem and non-iid problem on SFL with any types of DNN models. Compare with state-of-the-art FL methods, $S^2FL$ solves the problem that low-resource devices cannot train the complete model independently, enabling more low-resource devices to participate in the training. Compare with SFL method, $S^2FL$ can achieve higher inference accuracy and faster convergence rate.

\section{Our Approach}
Figure~\ref{fig:framework} presents the framework and workflow of S$^2$FL, which consists of two servers (i.e., the {\it Fed Server} and the {\it Main Server}) and multiple AIoT edge devices.
In S$^2$FL, the full global model is divided into three portions, i.e., the client-side model portion, the server-side model portion, and the shared model portion, which are maintained by the {\it Fed Server}, the {\it Main Server}, and both the {\it Fed Server}s and the {\it Main Server}s, respectively.
The {\it Fed Server} is used for the dispatching of client model portion and the aggregation of the client-side model and the shared model portions.
As shown in Figure~\ref{fig:framework}, the {\it Fed Server} includes a client time table, which maintains the average training time of each client. According to the client time table, the {\it Fed Server} splits a suitable size model portion from the shared model portion and dispatches it together with the client-side model portion as the client model portion to a specific AIoT device for model training.
The {\it Main Server} is responsible for training the latter part of the model. It receives the features and labels uploaded by the client devices, continuing the forward and backward propagation of the model, and finally returning the gradient to each device.


As shown in Figure \ref{fig:framework}, the workflow of our approach includes the following nine steps:
\begin{itemize}
\item[$\bullet$]\textbf{Step 1 (Model Selection):} The {\it {\it Fed Server}} chooses a client model portion $W_c$ for each selected device based on the {\it Client Time Table}. Each client model portion includes the complete client-side model portion and partial layers of the shared model portion.
\item[$\bullet$]\textbf{Step 2 (Model Dispatching): } The {\it {\it Fed Server}} dispatches the model portion $W_c$ to the selected devices.
\item[$\bullet$]\textbf{Step 3 (Feature Uploading): } Each device uses local data to perform forward propagation on the $W_c$ and upload the calculated features to the {\it {\it Main Server}}.
\item[$\bullet$]\textbf{Step 4 (Label Uploading): } Each device uploads the labels of local data to the {\it {\it Main Server}}.
\item[$\bullet$]\textbf{Step 5 (Data Balance): } When the {\it {\it Main Server}} has finished collecting features and labels, it uses the data balance-based training mechanism to group these features, and the {\it {\it Main Server}} will prepare the same number of copies of the server model portion $W_s$ as the number of groups, and the features in the same group will be input into the corresponding $W_s$ to calculate the loss. 
\item[$\bullet$]\textbf{Step 6 (Server Model Backward): } The {\it {\it Main Server}} aggregates all losses for each group and performs a backward propagation to update the corresponding server model portion $W_s$, and obtains the gradient of the features uploaded by each client device.
\item[$\bullet$]\textbf{Step 7 (Gradients Dispatching): } The {\it {\it Main Server}} sends the gradient to the client device.
\item[$\bullet$]\textbf{Step 8 (Client Model Backward): } After each device receives the gradient sent by the {\it {\it Main Server}}, it updates the client model portion $W_c$ with the gradient.
\item[$\bullet$]\textbf{Step 9 (Model Uploading):} Each client sends its model $W_c$ to the {\it {\it Fed Server}} for model aggregation.
\end{itemize}

\subsection{Adaptive Sliding Model Split Strategy}
In the AIoT scenario, edge devices possess varying computational capabilities and network bandwidth. We measure them respectively using the floating-point operations per second (FLOPS) and the transmission rates of each edge device. Assuming a complete model $W$ is divided into a client model portion $W_c$ and a server model portion $W_s$, where the size of $W_c$ is $|W_c|$, the output size of $W_c$'s final layer is $q$, and the floating-point computation counts for the forward and backward propagation of $W_c$ is $F_c$, while for $W_s$ it is denoted as $F_s$. The computation capability of the device is represented as $Comp_c$, with a transmission rate of $R$ and a local data quantity of $p$. The computation capability for the server is denoted as $Comp_s$. In our workflow, a complete training iteration starts with the {\it Fed Server} sending $W_c$ to the device and ends when the {\it Fed Server} receives the uploaded $W_c$ from the device. Under our assumptions, the time $T$ required for a single device to complete one round of training becomes: 
\begin{equation}
    T=\frac{2|W_c|+2pq}{R}+\frac{F_c}{Comp_c}+\frac{F_s}{Comp_s}.
\end{equation}

When each device trains the same $W_c$ and the same quantity $p$ of local data, high-performance devices with faster computational speed $Comp_c$ and higher transmission rate $R$ will be able to complete training more quickly, which causes high-performance devices to need to wait for other devices. To address this problem, the {\it Fed Server} will dynamically select the model portion that is appropriate for each device's computational resources so that the time $T$ spent on a training round is similar between devices. Specifically, in order to match devices with different computational resources, we set up $K$ possible split layers, in the first $K$ rounds of our approach, the {\it Fed Server} traverses all the split layers and uses one of them to divide $W$ into $W_c$ and $W_s$ within a same round, then it sends $W_c$ to all the devices and turns on a timer. When the {\it Fed Server} receives a $W_c$ uploaded by a device, it then records the time spent by that device this round in its own client time table. The client time table is a table for evaluating the computational resources of each device, which records for each device the time it takes to complete a training round when training different sizes of $W_c$. After the warm-up phase of round $K$, the {\it Fed Server} will select a subset of devices (assume $x$ devices) to participate in training in each round, and it will collect the time spent by these devices of training on each different $W_c$ from the client time table, and take the median of these times (total has $x*K$ times). Then, for each device, the {\it Fed Server} selects $W_c$ with the training time closest to the median as the client model portion used for this round of training. When the {\it Fed Server} receives the model uploaded by the device in this round, it dynamically updates the client time table based on the size of the client model portion and the training time.

\begin{figure}[h]
    \centering
    \includegraphics[width=0.5\textwidth]{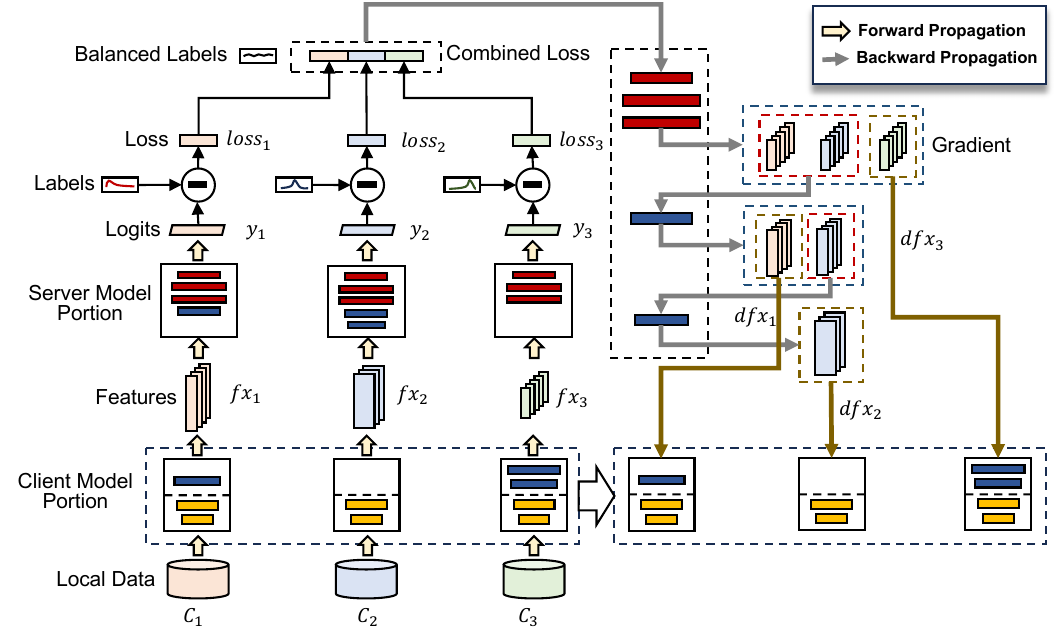}
    \caption{Data balance-based training mechanism}
    \label{union_train}
\end{figure}

\subsection{Data Balance-based Training Mechanism}
To address the problem of low inference accuracy caused by data heterogeneity, as shown in Figure~\ref{union_train}, we present a data balance-based training mechanism.
Assume that the dataset $D$ consists of data with $n$ categories, where $D$= $[D^{(1)},D^{(2)},D^{(3)},..., D^{(n)}]$ and $D^{(i)}$ represents the data set with the label $i$.

When the {\it Main Server} receives the intermediate features $fx$ and their labels $\mathcal L$ from the client device, it gets the data distribution of the client through $\mathcal L$. Next, the {\it Main Server} groups the $fx$ uploaded by devices.

Specifically, for $x$ clients participating in each round of training, the {\it Main Server} groups the $fx$ uploaded by clients whose combined data distribution is closest to the uniform distribution. In other words, we sum and normalize the data distribution of devices and then calculate the L2 distance from the uniform distribution as follows:
\begin{equation}
    \label{"balance loss"}
     Dist = \sqrt{\sum^n_{i=1} {(\frac{\sum_{c \in \mathcal{G}}|D^{(i)}_{c}|}{\sum_{p \in \mathcal{G}} |D_{p}|}-\frac{1}{n})^2}},
\end{equation}
where $\mathcal{G}$ denotes a group of devices, $|D^{(i)}_{c}|$ denotes the quantity of data with the label $i$ in client $c$, $|D_{p}|$ indicates total amount of data for the client $p$.
Note that the goal of our grouping is to minimize the distance.
After the grouping process, $fx$ in the same group inputs into the same server model portion $W_s$ for training. As shown in Figure \ref{union_train}, the intermediate features $fx_i$ within the same group are first input into $W_s$ for forward propagation to obtain logits $y_i$.
Then, the {\it Main Server} uses $y_i$ with $\mathcal{L}_i$ to calculate $loss_i$ using a specific loss function (e.g., cross entropy).
Next, the {\it Main Server} combines calculated losses in a group to generate the combined loss as follows:
\begin{equation}
    loss = \cup_{i \in group}loss_i.
\end{equation}
The {\it Main Server} performs backward propagation to update the server model portion $W_s$, calculates the gradient $dfx_i$ of each $fx_i$ based on the combined $loss$, and sends the $dfx_i$ to the corresponding device for local $W_c$ updating. The above process can be presented as follows:
\begin{equation}
\footnotesize
  W_s = W_s - \eta * \nabla(loss), dfx_i = \frac{\partial loss}{\partial fx_i}, Wc_i = Wc_i - \eta * dfx_i.
\end{equation}

\begin{algorithm}[h]
    \caption{Model aggregation}
    \label{alg:aggregation} 
    \KwIn{
    i) $C$, the client model set; 
    ii) $S$, the server model set;
    }
    \KwOut{
    $W$, the full model
    }
    \textbf{Aggregation} ($C$, $S$) \textbf{begin} \\
    $(|D_{1}|,|D_{2}|,..,|D_{x}|) \leftarrow$ Data size of each device \;
    \For{$layer$ in $W$} {
        \For{$Wc_i$ in $C$} {
            \If{$Wc_i$ has $layer$} {
                $W[layer] \leftarrow W[layer] + |D_{i}| \times Wc_i[layer]$\;
            }
            \Else{
                $W_{s}$ $\leftarrow$ the corresponding model to $Wc_i$ in $S$\;
                $W[layer] \leftarrow W[layer] + |D_{i}| \times W_{s}[layer]$\;
            }
            $weight[layer] \leftarrow weight[layer] + |D_{i}|$\;
        }
    }
    \For{$layer$ in $W$} {
        $W[layer] \leftarrow W[layer] / weight[layer]$
    }
    \textbf{return} $W$\;
    \textbf{end} 
\end{algorithm}

\subsection{Model Aggregation}
In S$^2$FL, the model is divided into three parts: client-side
model portion, shared model portion, and server-side model portion. The layer of the shared model portion can be split into the server model $W_s$ or client model $W_c$. As a result, S$^2$FL cannot use the simple weighted average to aggregate models like FedAvg. Algorithm \ref{alg:aggregation} presents our model aggregation method. Line 2 obtains the data size of $x$ devices participating in training. Lines 3-14 traverse each $layer$ of the full model $W$. For each client model portion $Wc_i$ in the client model set, if the $Wc_i$ has the $layer$, then the $Wc_i[layer]$ is used to participate in the model aggregation. If $W_c$ doesn't have the $layer$, we use the $layer$ in $W_s$ that corresponds to $Wc_i$ in the server model set to participate in model aggregation. In both cases, the weight is set to the data size of device $|D_{i}|$. Lines 15-17 average the aggregated model $W$.

\subsection{Implementation of Our Approach}
Algorithm~\ref{alg:impl} presents the implementation of our approach. Line 3 selects $x$ clients from all clients to participate in this round of training. 
Line 4 uses the adaptive sliding model split strategy to choose a suitable client model portion $Wc_i$ for each client $C_i$, making the training time of each client similar. 
Lines 5-8 represent each client performing forward propagation on their local data $D_i$ on $Wc_i$. The {\it Main Server} receives the features $fx_i$ and labels $\mathcal{L}_i$ sent by the clients. 
Line 9 initializes the client model set and server model set to store the client model portion and server model portion. It is worth noting that multiple models belonging to the same group in the client model set will correspond to a single model in the server model set.
Line 10 uses the data balance-based training mechanism to group $fx$ based on the data distribution of clients. Lines 11-23 represent the process of training for each group. 
For each $fx_i$ within the same group $\mathcal{G}$, line 14 inputs $fx_i$ into server model portion $W_s$ related to $\mathcal{G}$ for forward propagation, obtaining the logits $y_i$. 
Line 15 calculates the $loss_i$ for each client using $y_i$ and $\mathcal{L}_i$. 
Line 17 combines all $loss_i$ to generate $loss$. 
Line 18 uses the $loss$ for backward propagation to update $W_s$.Line 19 saves updated $W_s$ to the server model set.  Lines 20-22 calculate the gradient $dfx_i$ for each $fx_i$, and lines 24-28 indicate that the {\it Main Server} sends the gradient $dfx_i$ to the corresponding client $C_i$. 
The client uses $dfx_i$ to update the client model portion $Wc_i$ and then uploads $Wc_i$ to the {\it Fed Server}. {\it Fed Server} saves the $Wc_i$ into the client model set. 
Line 29 aggregates the models from both the client model set and the server model set to form the complete model $W$.

\begin{algorithm}[h]
    \caption{Our $S^2FL$ approach}
    \label{alg:impl} 
    \KwIn{
    i) $r$, maximum number of rounds; 
    ii) $C$, client set; 
    iii) $\eta$, learning rate;
    }
    \KwOut{
    $W$, the full model
    }
    $\textbf{S}^2\textbf{FL}$ ($r$, $C$, $\eta$) \textbf{begin} \\
    \For{$t$=1 \textbf{to} $r$} {
        $(C_1,..C_x) \leftarrow $ random choose $x$ clients from $C$\;
        $[Wc_{i} | i = 1..x] \leftarrow$ sliding\_model\_split()\;
        \For{$i$=1 \textbf{to} $x$} {
            $D_i$ $\leftarrow$ Local data in $C_i$\;
            $(fx_i,\mathcal{L}_i)$ $\leftarrow$ $Wc_i$($D_i$)\;
        }
        $cms,sms\leftarrow$ initialized client and server model set\;
        $\mathcal{G}\_set\leftarrow$ data\_balance\_mechanism($fx_i, \mathcal{L}_i$)\;
        /* parallel for */ \\
        \For{$\mathcal{G}$ \textbf{in} $\mathcal{G}\_set$} {
            \For{$i$ \textbf{in} $\mathcal{G}$} {
                $y_i$ = $W_s$($fx_i$)\;
                $loss_i = crossEntropy(y_i, \mathcal{L}_i)$\;
            }
            $loss = \cup_{i \in \mathcal{G}}loss_i$\;
            $W_s = W_s - \eta * \nabla(loss)$\;
            $sms$.append($W_s$)\;
            \For{$i$ \textbf{in} $\mathcal{G}$} {
                $dfx_i = \frac{\partial loss}{\partial fx_i}$\;
            }
        }
        \For{$i$=1 \textbf{to} $x$} {
            Send $dfx_i$ to $C_i$\;
            $Wc_i = Wc_i - \eta * dfx_i$\;
            $cms$.append($Wc_i$)\;
        }
        $W$ = Aggregation($cms$,$sms$)\;
    }
    \textbf{return}  $W$\;
    \textbf{end}\;
\end{algorithm}

\section{Convergence Analysis}
This section shows that $S^2FL$ converges to the global optimum at a rate of $O(1/t)$ for strongly convex and smooth functions and non-iid data.
Our convergence analysis is inspired by the convergence analysis in~\cite{li2019convergence,karimireddy2020scaffold}.

\subsection{Notations and Assumptions}
In our approach, we assume that both client model portion $W^c$ and server model portion $W^s$ are updated using Stochastic Gradient Descent (SGD). We use a single SGD for a timescale to divide the timeline into discrete slots $t$. After $t$ SGD updates, the client and server model portions are denoted as $W^c_t$ and $W^s_t$. We assume that the clients are categorized into $M$ groups through the \emph{data balance-based training mechanism}, and the clients in $m$-th group $\mathcal{G}_m$ are $\{m_1, m_2, ..., m_{n_m}\}$, then the process of the next SGD update in $m$-th group can be expressed as follows:
 \begin{equation}
 \footnotesize
 V^{cm_i}_{t+1} = W^{cm_i}_{t} - \eta_t * \nabla F_{m_{i}}(W^{sm}_{t} \oplus W^{cm_i}_{t}, \xi^{m_i}_{t})  \quad i=1..n_m,
\end{equation}
\begin{equation}
\footnotesize
V^{sm}_{t+1} = W^{sm}_{t} - \sum\limits_{i=1}^{n_m} \eta_t * \nabla F_{m_{i}}(W^{sm}_{t} \oplus W^{cm_i}_{t}, \xi^{m_i}_{t}).
\end{equation}

Here, we introduce temporary variables $V^{cm_i}_{t+1}$ and $V^{sm}_{t+1}$, which respectively store the intermediate results of each client model portion within the client group and the corresponding server model portion after one SGD update. $F_{m_{i}}$ represents the local objective of client $m_i$. $\eta_t$ is the learning rate and $\xi^{m_i}_{t}$ represents the local data used by client $m_i$ for SGD updates in the $t$-th round.

In $S^2FL$, model aggregation is performed once after every $E$ SGD updates,then we have:
\begin{equation}
\footnotesize
\begin{split}
w^{sm}_{t+1} \oplus w^{cm_i}_{t+1}=\left\{
\begin{array}{rl}
v^{sm}_{t+1} \oplus v^{cm_i}_{t+1}, &if (t + 1) \% E \neq 0 \\
\sum\limits_{m=1}^{M}\sum\limits_{i=1}^{n_m}p_{m_i}(v^{sm}_{t+1} \oplus v^{cm_i}_{t+1}), &if  (t + 1) \% E = 0\\
\end{array},
\right.  
\end{split}
\nonumber 
\end{equation}
where $p_{m_i} = \frac{|D_{m_i}|}{\sum\limits_{m=1}^{M}\sum\limits_{i=1}^{n_m}|D_{m_i}|}$ is the aggregation weight of client $m_i$.

To make our proof clear, the intermediate aggregation result of $w^{sm}_t \oplus w^{cm_i}_t$ and $v^{sm}_t \oplus v^{cm_i}_t$ are denoted as $\overline{w_t}$ and $\overline{v_t}$, and we define $\overline{g_t} = \sum\limits_{m=1}^{M}\sum\limits_{i=1}^{n_m}p_{m_i}\nabla F_{m_i}(w^{sm}_{t} \oplus w^{cm_i}_{t})$ and $g_t = \sum\limits_{m=1}^{M}\sum\limits_{i=1}^{n_m}p_{m_i}\nabla F_{m_i}(w^{sm}_{t} \oplus w^{cm_i}_{t}, \xi_t^k)$.

Inspired by~\cite{li2019convergence}, we make the following assumptions on the local objective functions $F_1$,$F_2$,$\cdots$ $F_{|c|}$ of each client:

\begin{assumption} \label{asm:smooth}
	$F_1, \cdots, F_{|c|}$ are all $L$-smooth:
	for all $v$ and $w$, $F_k(v)  \leq F_k(w) + (v - w)^T \nabla F_k(w) + \frac{L}{2}|v - w|_2^2$.
\end{assumption}

\begin{assumption} \label{asm:strong_cvx}
	$F_1, \cdots, F_N$ are all $\mu$-strongly convex:
	for all $v$ and $w$, $F_k(v)  \geq F_k(w) + (v - w)^T \nabla F_k(w) + \frac{\mu }{2} |v - w|_2^2$.
\end{assumption}

\begin{assumption} \label{asm:sgd_var}
	Let $\xi_t^k$ be sampled from the $k$-th device's local data uniformly at random.
	The variance of stochastic gradients in each device is bounded: 
 $$E \left\| \nabla F_k(w_t^s \oplus w_t^{ck},\xi_t^k) - \nabla F_k(w_t^s \oplus w_t^{ck}) \right\|^2 \le \sigma_k^2,$$ where $k=1,\cdots,|c|$.
\end{assumption}

\begin{assumption} \label{asm:sgd_norm}
	The expected squared norm of stochastic gradients is uniformly bounded, i.e., $E \left \| \nabla F_k(w_t^s \oplus w_t^{ck},\xi_t^k) \right \|^2  \le G^2$ for all $k=1,\cdots,|c|$ and $t=1,\cdots, T-1$.
\end{assumption}

 Let $F^*$ and $F_{m_i}^{\star}$ be the minimum value of global objective function $F$ and local objective function $F_{m_i}$. We can use $\Gamma = F^* -\sum\limits_{m=1}^{M}\sum\limits_{i=1}^{n_m}p_{m_i}F_{m_i}^{\star}$ for quantifying the degree of non-iid. Based on the above notations and assumptions, we can prove the Theorem 1.
\begin{theorem}
Assume that in $S^2FL$, the model aggregation is performed once after every $E$ SGD updates. Choose the $r = max\{8\frac{L}{\mu}, E\} - 1$ and the learning rate $\eta_t = \frac{2}{\mu(t + r)}$, We have
\begin{equation}
\begin{aligned}
E[F(\overline{w}_t)] - F^{\star} \leq \frac{L}{(r+t)}(\frac{2B}{\mu^2} + \frac{(r+1)}{2}E\Vert\overline{w}_1 - w^{\star}\Vert^2),
\end{aligned}
\end{equation}
where $\overline{w}_1$ is the initialized model,$w^{\star}$ is the optimal model, and $B = \sum\limits_{m=1}^{M}\sum\limits_{i=1}^{n_m}p_{m_i}^2\sigma_{m_i}^2 + 6L\Gamma + 8(E-1)^2G^2$.
\label{theo_1}
\end{theorem}

\subsection{Proofs of Key Lemmas}
Before proving the main theorem we proposed, we propose and prove some useful lemmas first.

\begin{lemma}[Results of one step SGD]
\label{lemma:conv_main} 
Assume Assumption~\ref{asm:smooth} and~\ref{asm:strong_cvx}. If $\eta_t\leq \frac{1}{4L}$, we have
\begin{equation}
\begin{aligned}
E\Vert \overline{v}_{t+1} - w^{\star}\Vert^2 \leq (1-\mu\eta_t)E\Vert \overline{w}_t - w^{\star}\Vert^2 +6L\eta_t^2\Gamma + 2E\\\sum\limits_{m=1}^{M}\sum\limits_{i=1}^{n_m}p_{m_i} \Vert \overline{w}_t - w^{sm}_{t} \oplus w^{cm_i}_{t} \Vert^2 + \eta_t^2 E \Vert g_t - \overline{g}_t \Vert^2,
\end{aligned}
\end{equation}
where $\Gamma = F^* - \sum\limits_{m=1}^{M}\sum\limits_{i=1}^{n_m}|p_{m_i}|F_{m_i}^{\star} \ge 0$.
\end{lemma}
\begin{proof}[proof of lemma \ref{lemma:conv_main}]
\label{proof:conv_main}

Notice that $\overline{v}_{t+1} = \overline{w}_t - \eta_tg_t$ and $E\Vert g_t\Vert=\overline{g}_t$,then
\begin{equation}
\label{equa:8}
\begin{aligned}
\Vert \overline{v}_{t+1} - w^{\star} \Vert^2 &= \Vert \overline{w}_{t} - \eta_tg_t  - w^{\star} - \eta_t\overline{g}_t + \eta_t\overline{g}_t \Vert^2\\&= \Vert\overline{w}_t - w^{\star} - \eta_t\overline{g}_t \Vert^2 + \eta_t^2\Vert g_t - \overline{g}_t \Vert^2 .
\end{aligned}
\end{equation}
We next focus on $\Vert\overline{w}_t - w^{\star} - \eta_t\overline{g}_t \Vert^2$ and split it into three parts
\begin{equation}
\label{equa:9}
\begin{aligned}
\Vert\overline{w}_t - w^{\star} - \eta_t\overline{g}_t \Vert^2 &= \Vert \overline{w}_t - w^{\star} \Vert^2 -2\eta_t\langle\overline{w}_t - w^{\star}, \overline{g}_t\rangle\\&+\eta_t^2\Vert\overline{g}_t^2\Vert^2.
\end{aligned}
\end{equation}
Base on the L-smoothness of $F_k(*)$ and the notation of $\overline{g}_t$,we bound $\eta_t^2\Vert\overline{g}_t^2\Vert^2$ as
\begin{equation}
\label{equa:10}
\begin{aligned}
\eta_t^2\Vert\overline{g}_t^2\Vert^2 &\leq \eta_t^2\sum\limits_{m=1}^{M}\sum\limits_{i=1}^{n_m}p_{m_i}\Vert\nabla F_{m_i} (w^{sm}_{t} \oplus w^{cm_i}_{t})\Vert^2 \\ &\leq 2L\eta_t^2\sum\limits_{m=1}^{M}\sum\limits_{i=1}^{n_m}p_{m_i}(F_{m_i}(w^{sm}_{t} \oplus w^{cm_i}_{t}) - F_{m_i}^{\star}).
\end{aligned}
\end{equation}
Next, we split $-2\eta_t\langle\overline{w}_t - w^{\star}, \overline{g}_t\rangle$ into two parts.
\begin{equation}
\label{equa:11}
\begin{aligned}
&-2\eta_t\langle\overline{w}_t - w^{\star}, \overline{g}_t\rangle \\ &= -2\eta_t\sum\limits_{m=1}^{M}\sum\limits_{i=1}^{n_m}\langle\overline{w}_t - w^{\star}, \nabla F_{m_i} (w^{sm}_{t} \oplus w^{cm_i}_{t})\rangle \\ &=-2\eta_t\sum\limits_{m=1}^{M}\sum\limits_{i=1}^{n_m}\langle\overline{w}_t -w^{sm}_{t} \oplus w^{cm_i}_{t} , \nabla F_{m_i} (w^{sm}_{t} \oplus w^{cm_i}_{t})\rangle \\ & -2\eta_t\sum\limits_{m=1}^{M}\sum\limits_{i=1}^{n_m}\langle w^{sm}_{t} \oplus w^{cm_i}_{t} - w^{\star} , \nabla F_{m_i} (w^{sm}_{t} \oplus w^{cm_i}_{t})\rangle.
\end{aligned}
\end{equation}
By Cauchy-Schwarz inequality and AM-GM inequality, we bound the first part.
\begin{equation}
\label{equa:12}
\begin{aligned}
&-2\langle\overline{w}_t -w^{sm}_{t} \oplus w^{cm_i}_{t} , \nabla F_{m_i} (w^{sm}_{t} \oplus w^{cm_i}_{t})\rangle \\ & \leq \frac{1}{\eta_t}\Vert \overline{w}_t - w^{sm}_{t} \oplus w^{cm_i}_{t} \Vert^2 + \eta_t \Vert \nabla F_{m_i} (w^{sm}_{t} \oplus w^{cm_i}_{t}) \Vert^2.
\end{aligned}
\end{equation}
By the $\mu$-strong convexity of $F_k(*)$, we bound the second part.
\begin{equation}
\label{equa:13}
\begin{aligned}
&-\langle w^{sm}_{t} \oplus w^{cm_i}_{t} - w^{\star} , \nabla F_{m_i} (w^{sm}_{t} \oplus w^{cm_i}_{t})\rangle  \leq \\ & -(F_{m_i} (w^{sm}_{t} \oplus w^{cm_i}_{t}) - F_{m_i}(w^{\star})) - \frac{\mu}{2} \Vert w^{sm}_{t} \oplus w^{cm_i}_{t} - w^{\star} \Vert^2.
\end{aligned}
\end{equation}
By combining equation \ref{equa:9} - \ref{equa:13}, the L-smoothness of $F_k(*)$ and the notation of $\overline{w}_t$.$\Vert\overline{w}_t - w^{\star} - \eta_t\overline{g}_t \Vert^2$ follows that
\begin{equation}
\label{equa:14}
\begin{aligned}
& \Vert\overline{w}_t - w^{\star} - \eta_t\overline{g}_t \Vert^2 \leq \Vert\overline{w}_t - w^{\star}\Vert^2 + \\ &2L\eta_t^2\sum\limits_{m=1}^{M}\sum\limits_{i=1}^{n_m}p_{m_i}(F_{m_i}(w^{sm}_{t} \oplus w^{cm_i}_{t}) - F_{m_i}^{\star})  + \eta_t\sum\limits_{m=1}^{M}\sum\limits_{i=1}^{n_m}\\ & (\frac{1}{\eta_t} \Vert \overline{w}_t - w^{sm}_{t} \oplus w^{cm_i}_{t} \Vert^2 + \eta_t \Vert \nabla F_{m_i} (w^{sm}_{t} \oplus w^{cm_i}_{t}) \Vert^2) \\ & - 2\eta_t\sum\limits_{m=1}^{M}\sum\limits_{i=1}^{n_m}p_{m_i}(F_{m_i} (w^{sm}_{t} \oplus w^{cm_i}_{t}) - F_{m_i}(w^{\star}) \\ & + \frac{\mu}{2} \Vert w^{sm}_{t} \oplus w^{cm_i}_{t} - w^{\star} \Vert^2) = (1 - \mu\eta_t)\Vert\overline{w}_t - w^{\star}\Vert^2 \\ & + \sum\limits_{m=1}^{M}\sum\limits_{i=1}^{n_m}p_{m_i}\Vert \overline{w}_t - w^{sm}_{t} \oplus w^{cm_i}_{t} \Vert^2 \\ &+4L\eta_t^2\sum\limits_{m=1}^{M}\sum\limits_{i=1}^{n_m}p_{m_i}(F_{m_i} (w^{sm}_{t} \oplus w^{cm_i}_{t}) - F_{m_i}^{\star})\\ & - 2\eta_t\sum\limits_{m=1}^{M}\sum\limits_{i=1}^{n_m}p_{m_i}((w^{sm}_{t} \oplus w^{cm_i}_{t}) - F_{m_i}(w^{\star})).
\end{aligned}
\end{equation}
Next, we define $\gamma_t = 2\eta_t(1-2L\eta_t)$ and use the notation $\Gamma = F^* - \sum\limits_{m=1}^{M}\sum\limits_{i=1}^{n_m}|p_{m_i}|F_{m_i}^{\star}$, we have
\begin{equation}
\label{equa:15}
\begin{aligned}
&4L\eta_t^2\sum\limits_{m=1}^{M}\sum\limits_{i=1}^{n_m} p_{m_i}(F_{m_i} (w^{sm}_{t} \oplus w^{cm_i}_{t}) - F_{m_i}^{\star}) \\ & - 2\eta_t\sum\limits_{m=1}^{M}\sum\limits_{i=1}^{n_m} p_{m_i}((w^{sm}_{t} \oplus w^{cm_i}_{t}) - F_{m_i}(w^{\star})) = \\ & -\gamma_t\sum\limits_{m=1}^{M}\sum\limits_{i=1}^{n_m}p_{m_i}(F_{m_i} (w^{sm}_{t} \oplus w^{cm_i}_{t}) - F^{\star}) + 4L\eta_t^2\Gamma.
\end{aligned}
\end{equation}
We use convexity and L-smoothness of $F_k(*)$ and the AM-GM inequality to bound $\sum\limits_{m=1}^{M}\sum\limits_{i=1}^{n_m}p_{m_i}(F_{m_i} (w^{sm}_{t} \oplus w^{cm_i}_{t}) - F^{\star})$.
\begin{equation}
\label{equa:16}
\begin{aligned}
&\sum\limits_{m=1}^{M}\sum\limits_{i=1}^{n_m}p_{m_i}(F_{m_i} (w^{sm}_{t} \oplus w^{cm_i}_{t}) - F^{\star})) = \sum\limits_{m=1}^{M}\sum\limits_{i=1}^{n_m}p_{m_i}\\ &(F_{m_i} (w^{sm}_{t} \oplus w^{cm_i}_{t}) - F_{m_i}({\overline{w}_t}) + \sum\limits_{m=1}^{M}\sum\limits_{i=1}^{n_m}p_{m_i}(F_{m_i}(\overline{w}_t) - F^{\star}) \\ & \geq \sum\limits_{m=1}^{M}\sum\limits_{i=1}^{n_m}p_{m_i}\langle \nabla F_{m_i}(\overline{w}_t), w^{sm}_{t} \oplus w^{cm_i}_{t} - \overline{w}_t) \rangle + (F(\overline{w}_t) - \\ & F^{\star})  \geq -\frac{1}{2}\sum\limits_{m=1}^{M}\sum\limits_{i=1}^{n_m}p_{m_i}[\eta_t\Vert\nabla F_{m_i}(\overline{w}_t)\Vert^2 + \frac{1}{\eta_t}\Vert w^{sm}_{t} \oplus  w^{cm_i}_{t} - \\ & \overline{w}_t\Vert^2]  + (F(\overline{w}_t) - F^{\star}) \geq -\sum\limits_{m=1}^{M}\sum\limits_{i=1}^{n_m}p_{m_i}[\eta_tL(F_{m_i}(\overline{w}_t) - \\ & F_k^{\star}) +  \frac{1}{2\eta_t}\Vert w^{sm}_{t} \oplus  w^{cm_i}_{t} - \overline{w}_t\Vert^2] + (F(\overline{w}_t) - F^{\star}).
\end{aligned}
\end{equation}
Last, by combining equation \ref{equa:8} and \ref{equa:14} - \ref{equa:16}, we can use the following facts to complete the proof: (1) $\eta_tL -1 \leq -\frac{3}{4}$ (2)$\sum\limits_{m=1}^{M}\sum\limits_{i=1}^{n_m}p_{m_i}(F(\overline{w}_t) - F^{\star}) = F(\overline{w}_t) - F^{\star} \geq 0$ (3)$\Gamma \geq 0$ (4)$\gamma_t\eta_t \leq 2 $ and $\frac{\gamma_t}{2\eta_t} \leq 1$.
\begin{equation}
\label{equa:17}
\begin{aligned}
&-\gamma_t\sum\limits_{m=1}^{M}\sum\limits_{i=1}^{n_m}p_{m_i}(F_{m_i} (w^{sm}_{t} \oplus w^{cm_i}_{t}) - F^{\star}) \leq \gamma_t\sum\limits_{m=1}^{M}\sum\limits_{i=1}^{n_m}\\&p_{m_i}[\eta_tL(F_{m_i}(\overline{w}_t) - F_k^{\star}) + \frac{1}{2\eta_t}\Vert w^{sm}_{t} \oplus  w^{cm_i}_{t} - \overline{w}_t\Vert^2] - \gamma_t\\ &(F(\overline{w}_t) - F^{\star}) \leq 2L\eta_t^2\Gamma + \sum\limits_{m=1}^{M}\sum\limits_{i=1}^{n_m}p_{m_i} \Vert \overline{w}_t - w^{sm}_{t} \oplus w^{cm_i}_{t} \Vert^2,
\end{aligned}
\end{equation}

\begin{equation}
\label{equa:18}
\begin{aligned}
&\Vert \overline{v}_{t+1} - w^{\star} \Vert^2 = \Vert\overline{w}_t - w^{\star} - \eta_t\overline{g}_t \Vert^2 + \eta_t^2\Vert g_t - \overline{g}_t \Vert^2 \\ & \leq (1 - \mu\eta_t)\Vert\overline{w}_t - w^{\star}\Vert^2 + \sum\limits_{m=1}^{M}\sum\limits_{i=1}^{n_m}p_{m_i}\Vert \overline{w}_t - w^{sm}_{t} \oplus w^{cm_i}_{t} \Vert^2 \\ & -\gamma_t\sum\limits_{m=1}^{M}\sum\limits_{i=1}^{n_m}p_{m_i}(F_{m_i} (w^{sm}_{t} \oplus w^{cm_i}_{t}) - F^{\star}) + 4L\eta_t^2\Gamma \\ & + \eta_t^2\Vert g_t - \overline{g}_t \Vert^2  \leq (1-\mu\eta_t)E\Vert \overline{w}_t - w^{\star}\Vert^2 + 6L\eta_t^2\Gamma + \\ & 2E\sum\limits_{m=1}^{M}\sum\limits_{i=1}^{n_m}p_{m_i}\Vert \overline{w}_t - w^{sm}_{t} \oplus w^{cm_i}_{t} \Vert^2 + \eta_t^2 E \Vert g_t - \overline{g}_t \Vert^2.
\end{aligned}
\end{equation}
\end{proof}

\begin{lemma}[Bounding the variance]
Assume Assumption~\ref{asm:sgd_var} holds. It follows that
	\label{lemma:conv_variance}
\begin{equation}
\begin{aligned}
E\Vert g_t- \overline{g}_t \Vert^2 \leq \sum\limits_{m=1}^{M}\sum\limits_{i=1}^{n_m}p_{m_i}^2\sigma_{m_i}^2.
\end{aligned}
\end{equation}
\end{lemma}
\begin{proof}[proof of lemma \ref{lemma:conv_variance}]
\label{proof:conv_variance}
From Assumption \ref{asm:sgd_var}, the variance of stochastic gradients in each device is bounded by $\sigma_k^2$, then
\begin{equation}
\label{equa:20}
\begin{aligned}
&E\Vert g_t - \overline{g}_t \Vert^2 =  E \Vert\sum\limits_{m=1}^{M}\sum\limits_{i=1}^{n_m}p_{m_i}(\nabla F_{m_i}(w^{sm}_{t} \oplus w^{cm_i}_{t}, \xi_t^k) \\&- \nabla F_{m_i}(w^{sm}_{t} \oplus w^{cm_i}_{t})) \Vert^2 = \sum\limits_{m=1}^{M}\sum\limits_{i=1}^{n_m}p_{m_i}^2E\Vert\nabla F_{m_i}(w^{sm}_{t} \\& \oplus w^{cm_i}_{t}, \xi_t^k) - \nabla F_{m_i}(w^{sm}_{t} \oplus w^{cm_i}_{t}) \Vert^2 \leq \sum\limits_{m=1}^{M}\sum\limits_{i=1}^{n_m}p_{m_i}^2\sigma_{m_i}^2.
\end{aligned}
\end{equation}
\end{proof}

\begin{lemma}[Bounding the divergence of ${w_t^k}$]
	\label{lemma:conv_diversity}
	Assume Assumption~\ref{asm:sgd_norm}, that $\eta_t$ is non-increasing and $\eta_{t} \le 2 \eta_{t+E}$ for all $t\geq 0$. It follows that
\begin{equation}
\begin{aligned}
 E[\sum\limits_{m=1}^{M}\sum\limits_{i=1}^{n_m}p_{m_i} \Vert \overline{w}_t - w^{sm}_{t} \oplus w^{cm_i}_{t}\Vert^2] \leq 4\eta^2_t(E-1)^2G^2.
\end{aligned}
\end{equation}
\begin{proof}[proof of lemma \ref{lemma:conv_diversity}]
\label{proof:conv_diversity}
Since $S^2FL$ performs model aggregation after every $E$ SGD update, for any round $t$, there must be a model aggregation process between round $t-E+1$ and round $t$.That is, there exists a $t_0$ satisfying $t - t_0 \leq E-1$ and $w^k_{t_0} = \overline{w}_{t_0}$ for $k$ from 1 to $|c|$. Then combining the assumption \ref{asm:sgd_norm} and $\eta_{t} \le 2 \eta_{t+E}$ for all $t\geq 0$, we have
\begin{equation}
\label{equa:22}
\begin{aligned}
&E[\sum\limits_{m=1}^{M}\sum\limits_{i=1}^{n_m}p_{m_i} \Vert \overline{w}_t - w^{sm}_{t} \oplus w^{cm_i}_{t}\Vert^2] \\& = E[\sum\limits_{m=1}^{M}\sum\limits_{i=1}^{n_m}p_{m_i} \Vert (w^{sm}_{t} \oplus w^{cm_i}_{t} - \overline{w}_{t_0}) - (\overline{w}_t -\overline{w}_{t_0})]\Vert^2 \\& \leq  E[\sum\limits_{m=1}^{M}\sum\limits_{i=1}^{n_m}p_{m_i} \Vert (w^{sm}_{t} \oplus w^{cm_i}_{t} - \overline{w}_{t_0})\Vert^2]
\\& \leq \sum\limits_{m=1}^{M}\sum\limits_{i=1}^{n_m}p_{m_i}E\sum\limits_{t=t_0}^{t-1}(E-1)\eta_t^2\Vert\nabla F_{m_i}(w^{sm}_{t} \oplus w^{cm_i}_{t}, \xi_t^{m_i}) \Vert^2
\\& \leq \sum\limits_{m=1}^{M}\sum\limits_{i=1}^{n_m}p_{m_i}\sum\limits_{t=t_0}^{t-1}(E-1)\eta_{t_0}^2G^2
\\& \leq \sum\limits_{m=1}^{M}\sum\limits_{i=1}^{n_m}p_{m_i}(E-1)^2\eta_{t_0}^2G^2
\\& \leq 4\eta^2_t(E-1)^2G^2.
\end{aligned}
\end{equation}
\end{proof}
\end{lemma}
\subsection{Proof of Theorem \ref{theo_1}}
Using the previous three lemmas, we can prove the theorem \ref{theo_1} showing that $S^2FL$ converges to the global optimum at a rate of $O(1/t)$.
\begin{proof}
\label{proof:main_theorem}
Note that no matter whether the current round $t$ is an aggregation round or not, $\overline{w}_t = \overline{v}_t$ always holds. We let $\triangle_t = E\Vert\overline{w}_t - w^{\star}\Vert^2$. From Lemma \ref{lemma:conv_main}, we have
\begin{equation}
\begin{aligned}
&\triangle_{t+1} \leq (1-\mu\eta_t)\triangle_{t} +6L\eta_t^2\Gamma + 2E\sum\limits_{m=1}^{M}\sum\limits_{i=1}^{n_m}p_{m_i}\\& \Vert \overline{w}_t - w^{sm}_{t} \oplus w^{cm_i}_{t} \Vert^2 + \eta_t^2 E \Vert g_t - \overline{g}_t \Vert^2.
\end{aligned}
\end{equation}
Next, we use Lemma \ref{lemma:conv_variance} to bound $E \Vert g_t - \overline{g}_t \Vert^2$ and use Lemma \ref{lemma:conv_diversity} to bound $E\sum\limits_{m=1}^{M}\sum\limits_{i=1}^{n_m}p_{m_i} \Vert \overline{w}_t - w^{sm}_{t} \oplus w^{cm_i}_{t} \Vert^2$,
\begin{equation}
\begin{aligned}
\triangle_{t+1} &\leq (1-\mu\eta_t)\triangle_{t} + 6L\eta_t^2\Gamma + 8\eta_t^2(E-1)^2G^2 \\ & + \eta_t^2\sum\limits_{m=1}^{M}\sum\limits_{i=1}^{n_m}p_{m_i}^2\sigma_{m_i}^2 = (1-\mu\eta_t)\triangle_{t} + \eta_t^2B,
\end{aligned}
\end{equation}
where B = $\sum\limits_{m=1}^{M}\sum\limits_{i=1}^{n_m}p_{m_i}^2\sigma_{m_i}^2 + 6L\Gamma + 8(E-1)^2G^2$.

We set $\eta_t = \frac{\beta}{t + r}$ for some $\beta > \frac{1}{\mu}$ and $r > 0$ such that $\eta_1 \leq min\{\frac{1}{\mu}, \frac{1}{4L}\}$ and $\eta_t \leq 2\eta_{t+E}$.Then, we will prove $\triangle_t \leq \frac{v}{r+t}$ by induction, where $v = max\{\frac{\beta^2B}{\beta\mu-1}, (r+1)\triangle_1\}$.

Firstly, the definition of $v$ ensures that $\triangle_t \leq \frac{v}{r+t}$ holds for $t = 1$. Assume it holds for some round $t$, the $\triangle_{t+1}$ follows that
\begin{equation}
\begin{aligned}
\triangle_{t+1} & \leq (1-\mu\eta_t)\triangle_t + \eta_t^2B \\
& \leq (1-\frac{\beta\mu}{t+r})\frac{v}{t+r}+\frac{\beta^2B}{(t+r)^2} \\
& = \frac{t+r-1}{(t+r)^2}v + [\frac{\beta^2B}{(t+r)^2} - \frac{\beta\mu - 1}{(t+r)^2}v] \\
& \leq \frac{v}{t+r+1}.
\end{aligned}
\end{equation}
By the L-smoothness of $F(*)$, we have
\begin{equation}
E[F(\overline{w}_t)] - F^{\star} \leq \frac{L}{2}\triangle_t \leq \frac{L}{2}\frac{v}{r+t}.
\end{equation}
Specifically, we choose $\beta = \frac{2}{\mu}, r = max\{8\frac{L}{\mu}, E\} - 1$. We can obtain the following two equations
\begin{equation}
\begin{aligned}
&v = max\{\frac{\beta^2B}{\beta\mu-1}, (r+1)\triangle_1\} \leq \frac{\beta^2B}{\beta\mu-1} + (r+1)\triangle_1\\ & \leq \frac{4B}{\mu^2} + (r+1)\triangle_1,
\end{aligned}
\end{equation}
\begin{equation}
\begin{aligned}
E[F(\overline{w}_t)] - F^{\star} \leq \frac{L}{2}\frac{v}{r+t} \leq \frac{L}{(r+t)}(\frac{2B}{\mu^2} + \frac{(r+1)}{2}\triangle_1).
\end{aligned}
\end{equation}
\end{proof}

\section{Experimental Results}
To evaluate the effectiveness of $S^2FL$, we implemented our approach using the PyTorch framework. We compared $S^2FL$ with the classical FL method FedAvg and the vanilla SFL. We select three different split layers on each model $W$, dividing $W$ into three different model portion tuples:[$(Wc_1,Ws_1)$, $(Wc_2,Ws_2)$, $(Wc_3,Ws_3)$] where $size(Wc_1)<size(Wc_2)<size(Wc_3)<size(W)$. We used the Python tool libraries \emph{thop} to count the size of these model portions and the FLOPs needed to perform one forward and backward propagation on a single input, as shown in Figure \ref{modelsetting}. This Figure shows that the smaller model portion $W_c$ has a smaller size and fewer FLOPs compared to the full model $W$, and this disparity appears to be more pronounced for complex models. Devices train the complete model $W$ in FedAvg and train the largest client model portion $Wc_3$ in SFL. Our approach uses the adaptive sliding model split strategy to choose a suitable client model portion for each device. To enable fair comparison, we set the batch size to 128 and used an SGD optimizer with a learning rate of 0.01 for both all the baselines and $S^2FL$. All the experimental results were obtained from an Ubuntu workstation with an Intel i9 CPU, 64GB memory, and an NVIDIA RTX 3090 GPU.

\subsection{Experimental Settings}
1) \emph{Settings of Devices}: We set the FLOPS and transfer rate of each device to measure the computing power and network bandwidth of these devices. In the experiment, we assume that there are three different FLOPS and three different transfer rates. Table \ref{device setting} shows the settings of FLOPS and transfer rates. Each device has a specific quality of FLOPS and transfer rate in this table. Because FLOPS and transfer rate are not correlated,  there are a total of 9 different kinds of devices in our experiment. We assume that the server has ultra-high computing resources, its FLOPS is set to $5\times10^{10}$, and the transfer rate is set to $1\times10^7$.

\begin{figure}[h]
    \centering
    \subfloat[Sizes of different model]{
    \includegraphics[width=0.23\textwidth]{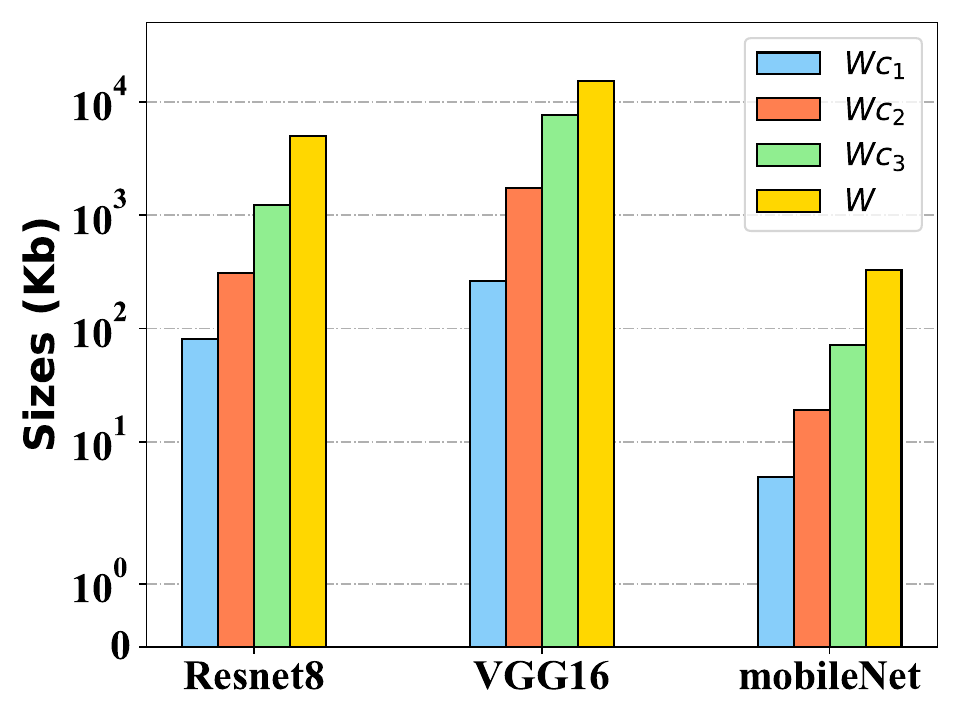}
        \label{fig:modelsetting_size}
    }
    \subfloat[FLOPs of different model]{
    \includegraphics[width=0.23\textwidth]{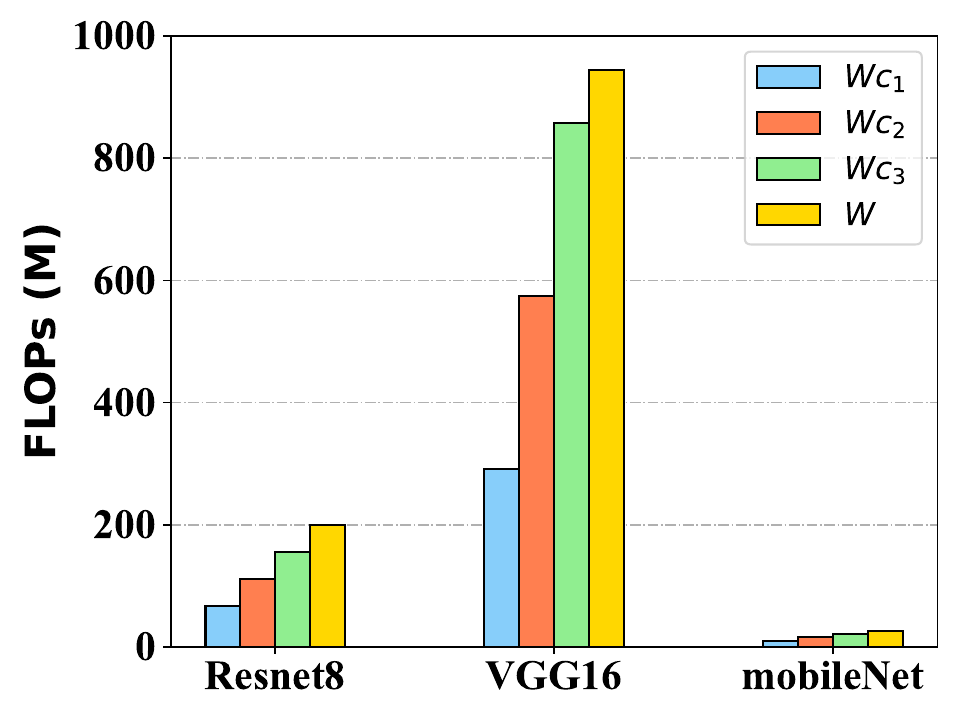}
        \label{fig:modelsetting_flops}
    }
    \caption{Comparison of size and FLOPs of different model portions}
    \label{modelsetting}
\end{figure}

\begin{table}[h]
\centering
\caption{Devices settings}
\label{device setting}
\begin{tabular}{|c|c|c|}
\hline
Quality & FLOPS Settings & Transfer Rate Settings \\ \hline\hline
Low    & $5\times10^9$  & $1\times10^6$          \\ \hline
Mid     & $1\times10^{10}$  & $2\times10^6$          \\ \hline
High     & $2\times10^{10}$  & $5\times10^6$          \\ \hline
\end{tabular}
\end{table}

2) \emph{Settings of Datasets and Models:} We compared the
performance of our approach and two baselines on four well-known datasets, i.e., CIFAR-10, CIFAR-100\cite{krizhevsky2009learning}, ImageNet\cite{deng2009imagenet} and FEMNIST \cite{LEAF}. To investigate the performance of our approach on the non-IID scenarios, we use the Dirichlet distribution to control the data heterogeneity between devices on datasets CIFAR-10, CIFAR-100, and ImageNet and set a parameter $a$ to indicate the degree of heterogeneity, and for datasets FEMNIST, we use its natural no-IID distribution. For experiments on datasets CIFAR-10, CIFAR-100, and ImageNet, we assume that there are 100 AIoT devices, and 10 devices are randomly selected for local training in each training round. However, for dataset FEMNIST, there are a total of 180 devices, and each training round involves 18 devices. In addition, to show the pervasiveness of our approach, we experimented on three different DNN models(i.e., ResNet-8\cite{he2016deep}, VGG-16\cite{simonyan2014very}, and mobileNet\cite{howard2017mobilenets}).
\begin{figure}[h]
    \centering
    \subfloat[$a = 0.1$]{
    \includegraphics[width=0.23\textwidth]{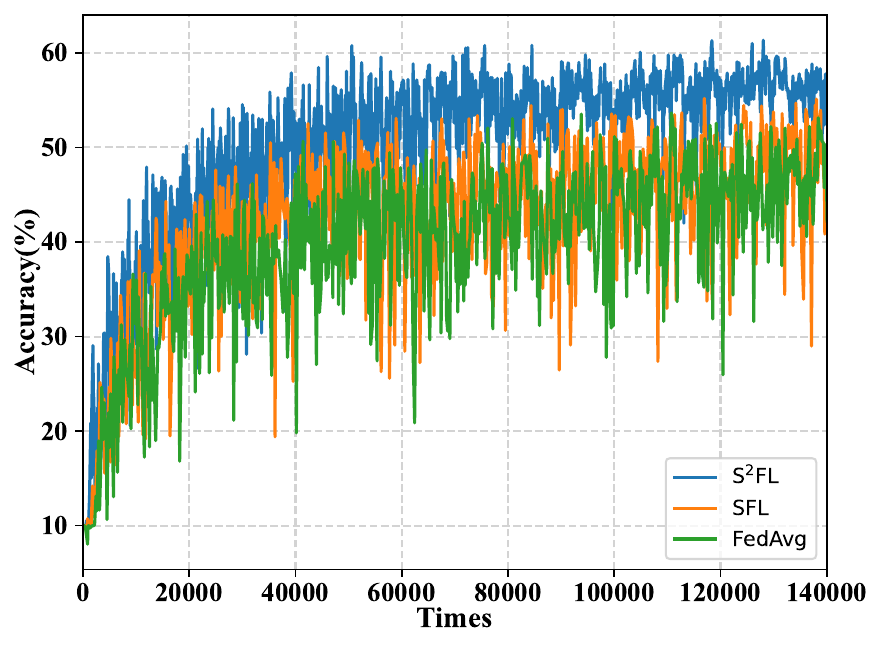}
        \label{fig:subfig1}
    }
    \hspace{-0.2in}
    \subfloat[$a = 0.5$]{
    \includegraphics[width=0.23\textwidth]{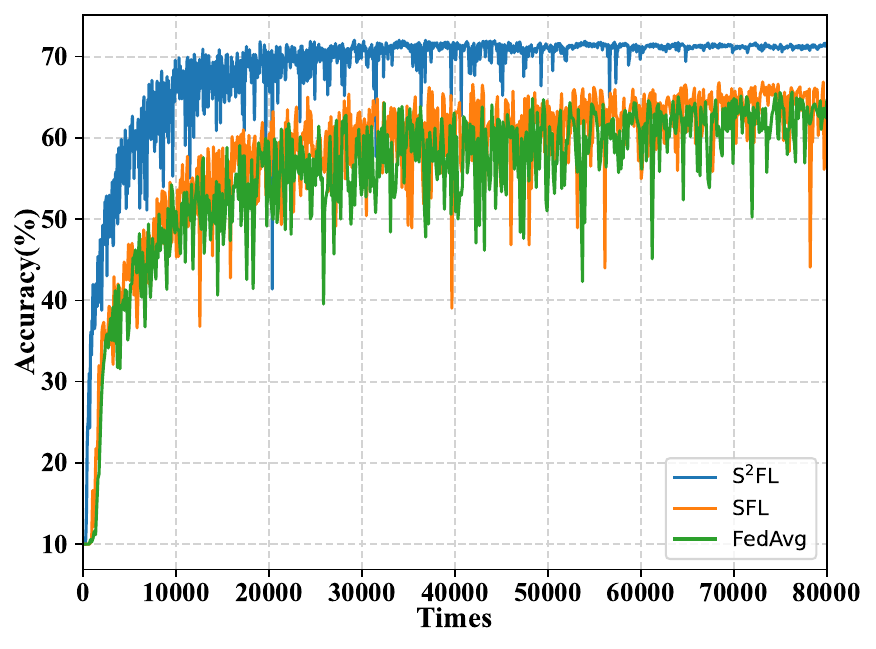}
        \label{fig:subfig2}
    }
    \hfill
    \subfloat[$a = 1.0$]{
    \includegraphics[width=0.23\textwidth]{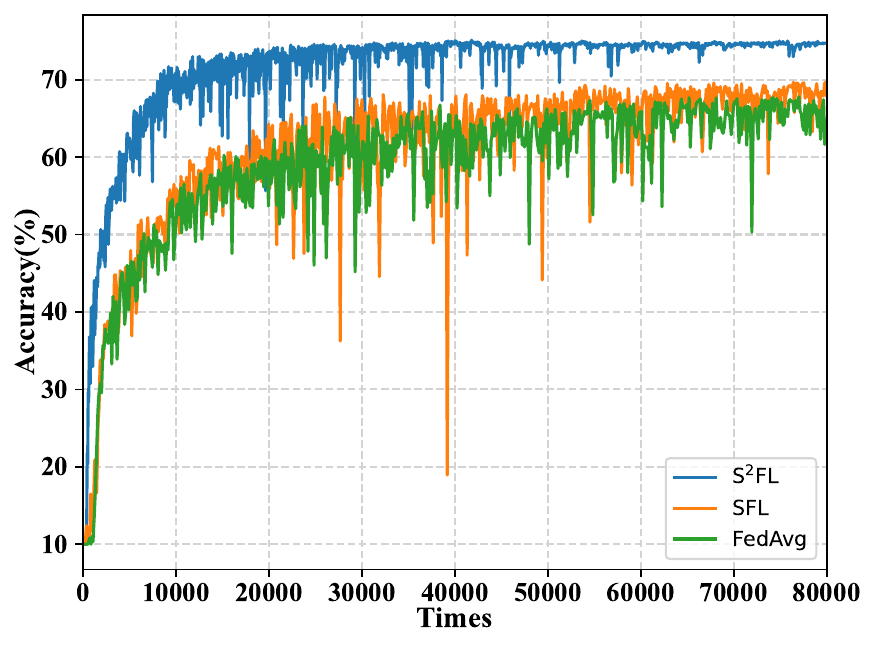}
        \label{fig:subfig3}
    }
    \hspace{-0.2in}
    \subfloat[$IID$]{
    \includegraphics[width=0.23\textwidth]{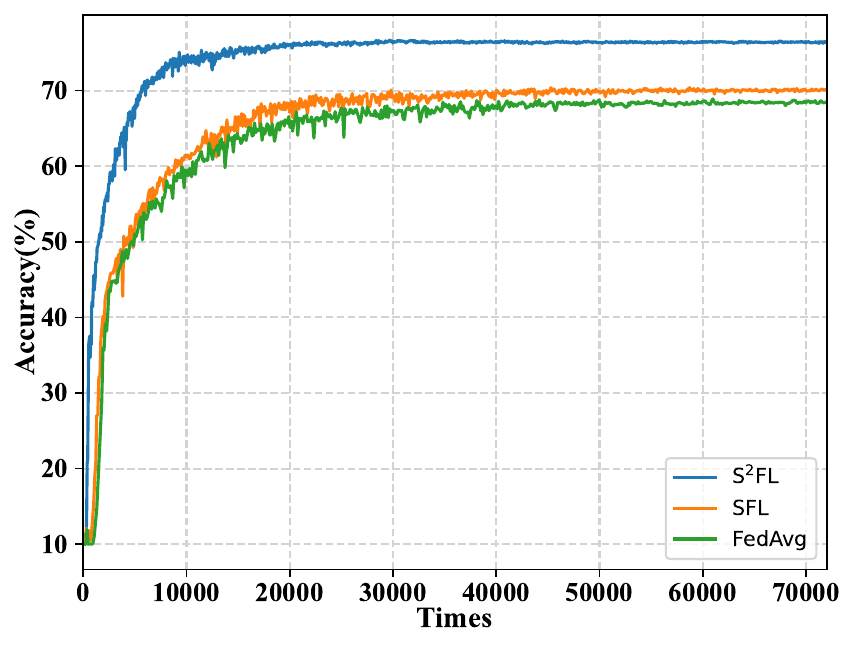}
        \label{fig:subfig4}
    }
    \caption{Training process of $S^2FL$ and two baselines on CIFAR-10}
    \label{test_acc}
\end{figure}

\begin{table*}[h]
\centering
\caption{Comparison of test accuracy}
\label{test acc}
\scalebox{1.0}{
\begin{tabular}{|c|c|c|ccc|}
\hline
\multirow{2}{*}{Dataset}    & \multirow{2}{*}{Model}     & \multirow{2}{*}{\begin{tabular}[c]{@{}c@{}}Hetero. \\      Settings\end{tabular}} & \multicolumn{3}{c|}{Test Accuracy(\%)}                                                      \\ \cline{4-6} 
                            &                            &                                                                                         & \multicolumn{1}{c|}{FedAvg}         & \multicolumn{1}{c|}{SFL}             & $S^2FL$(ours)     \\ \hline\hline
\multirow{12}{*}{CIFAR-10}  & \multirow{4}{*}{ResNet8}   & $a = 0.1$                                                                               & \multicolumn{1}{c|}{$48.65\pm1.16$} & \multicolumn{1}{c|}{$49.13\pm1.70$}  & $\textbf{51.87}\pm\textbf{1.70}$ \\
                            &                            & $a = 0.5$                                                                               & \multicolumn{1}{c|}{$56.96\pm0.31$} & \multicolumn{1}{c|}{$58.15\pm0.32$}  & $\textbf{60.43}\pm\textbf{0.38}$ \\
                            &                            & $a = 1.0$                                                                               & \multicolumn{1}{c|}{$58.72\pm0.27$} & \multicolumn{1}{c|}{$60.48\pm0.20$}  & $\textbf{62.78}\pm\textbf{0.19}$ \\
                            &                            & $IID$                                                                                   & \multicolumn{1}{c|}{$59.52\pm0.15$} & \multicolumn{1}{c|}{$62.02\pm0.07$}  & $\textbf{64.08}\pm\textbf{0.09}$ \\ \cline{2-6} 
                            & \multirow{4}{*}{VGG16}     & $a = 0.1$                                                                               & \multicolumn{1}{c|}{$51.61\pm1.73$} & \multicolumn{1}{c|}{$52.50\pm3.21$}   & $\textbf{58.94}\pm\textbf{1.58}$ \\
                            &                            & $a = 0.5$                                                                               & \multicolumn{1}{c|}{$65.53\pm0.27$} & \multicolumn{1}{c|}{$66.04\pm0.35$}  & $\textbf{72.45}\pm\textbf{0.21}$ \\
                            &                            & $a = 1.0$                                                                               & \multicolumn{1}{c|}{$68.67\pm0.29$} & \multicolumn{1}{c|}{$68.96\pm0.36$}  & $\textbf{74.65}\pm\textbf{0.20}$ \\
                            &                            & $IID$                                                                                   & \multicolumn{1}{c|}{$68.57\pm0.11$} & \multicolumn{1}{c|}{$70.15\pm0.06$}  & $\textbf{75.12}\pm\textbf{0.09}$ \\ \cline{2-6} 
                            & \multirow{4}{*}{mobileNet} & $a = 0.1$                                                                               & \multicolumn{1}{c|}{$35.68\pm0.62$} & \multicolumn{1}{c|}{$36.68\pm0.42$}  & $\textbf{38.19}\pm\textbf{0.83}$ \\
                            &                            & $a = 0.5$                                                                               & \multicolumn{1}{c|}{$47.85\pm0.55$} & \multicolumn{1}{c|}{$49.01\pm0.41$}  & $\textbf{52.80}\pm\textbf{0.26}$  \\
                            &                            & $a = 1.0$                                                                               & \multicolumn{1}{c|}{$52.29\pm0.57$} & \multicolumn{1}{c|}{$52.08\pm0.35$}  & $\textbf{58.33}\pm\textbf{0.34}$ \\
                            &                            & $IID$                                                                                   & \multicolumn{1}{c|}{$55.56\pm0.15$} & \multicolumn{1}{c|}{$55.21\pm0.08$}  & $\textbf{61.61}\pm\textbf{0.13}$ \\ \hline
\multirow{12}{*}{CIFAR-100} & \multirow{4}{*}{ResNet8}   & $a = 0.1$                                                                               & \multicolumn{1}{c|}{$29.72\pm0.30$}  & \multicolumn{1}{c|}{$27.15\pm0.15$}  & $\textbf{34.90}\pm\textbf{0.23}$ \\
                            &                            & $a = 0.5$                                                                               & \multicolumn{1}{c|}{$31.72\pm0.09$} & \multicolumn{1}{c|}{$28.52\pm0.15$}  & $\textbf{38.31}\pm\textbf{0.10}$ \\
                            &                            & $a = 1.0$                                                                               & \multicolumn{1}{c|}{$31.40\pm0.09$} & \multicolumn{1}{c|}{$28.72\pm0.14$}  & $\textbf{38.22}\pm\textbf{0.21}$ \\
                            &                            & $IID$                                                                                   & \multicolumn{1}{c|}{$31.21\pm0.06$} & \multicolumn{1}{c|}{$28.19\pm0.09$}  & $\textbf{38.61}\pm\textbf{0.08}$ \\ \cline{2-6} 
                            & \multirow{4}{*}{VGG16}     & $a = 0.1$                                                                               & \multicolumn{1}{c|}{$33.24\pm0.76$} & \multicolumn{1}{c|}{$30.55\pm0.42$} & $\textbf{44.04}\pm\textbf{0.44}$ \\
                            &                            & $a = 0.5$                                                                               & \multicolumn{1}{c|}{$31.22\pm0.19$} & \multicolumn{1}{c|}{$27.85\pm0.21$}  & $\textbf{44.34}\pm\textbf{0.18}$ \\
                            &                            & $a = 1.0$                                                                               & \multicolumn{1}{c|}{$31.90\pm0.17$} & \multicolumn{1}{c|}{$28.31\pm0.28$}  & $\textbf{44.49}\pm\textbf{0.13}$ \\
                            &                            & $IID$                                                                                   & \multicolumn{1}{c|}{$32.21\pm0.05$} & \multicolumn{1}{c|}{$29.48\pm0.09$} & $\textbf{44.60}\pm\textbf{0.15}$ \\ \cline{2-6} 
                            & \multirow{4}{*}{mobileNet} & $a = 0.1$                                                                               & \multicolumn{1}{c|}{$16.69\pm0.24$} & \multicolumn{1}{c|}{$15.40\pm0.21$} & $\textbf{21.15}\pm\textbf{0.13}$ \\
                            &                            & $a = 0.5$                                                                               & \multicolumn{1}{c|}{$18.36\pm0.10$} & \multicolumn{1}{c|}{$15.51\pm0.05$}  & $\textbf{23.31}\pm\textbf{0.09}$ \\
                            &                            & $a = 1.0$                                                                               & \multicolumn{1}{c|}{$18.24\pm0.10$} & \multicolumn{1}{c|}{$16.24\pm0.09$}  & $\textbf{23.30}\pm\textbf{0.09}$ \\
                            &                            & $IID$                                                                                   & \multicolumn{1}{c|}{$19.10\pm0.07$} & \multicolumn{1}{c|}{$16.74\pm0.10$}  & $\textbf{23.17}\pm\textbf{0.09}$ \\ \hline
\multirow{12}{*}{ImageNet} & \multirow{4}{*}{ResNet8}   & $a = 0.1$                                                                               & \multicolumn{1}{c|}{$35.94\pm0.14$}  & \multicolumn{1}{c|}{$36.43\pm0.21$}  &  $\textbf{37.10}\pm\textbf{0.17}$ \\
                            &                            & $a = 0.5$                                                                               & \multicolumn{1}{c|}{$38.53\pm0.16$} & \multicolumn{1}{c|}{$39.04\pm0.13$}  & $\textbf{39.83}\pm\textbf{0.78}$ \\
                            &                            & $a = 1.0$                                                                               & \multicolumn{1}{c|}{$39.21\pm0.19$} & \multicolumn{1}{c|}{$40.35\pm0.08$}  &  $\textbf{41.61}\pm\textbf{0.28}$ \\
                            &                            & $IID$                                                                                   & \multicolumn{1}{c|}{$39.01\pm0.13$} & \multicolumn{1}{c|}{$39.80\pm0.09$}  &  $\textbf{41.83}\pm\textbf{0.15}$ \\ \cline{2-6} 
                            & \multirow{4}{*}{VGG16}     & $a = 0.1$                                                                               & \multicolumn{1}{c|}{$33.52\pm0.66$} & \multicolumn{1}{c|}{$36.42\pm0.90$} & $\textbf{39.92}\pm\textbf{0.22}$ \\
                            &                            & $a = 0.5$                                                                               & \multicolumn{1}{c|}{$36.79\pm0.96$} & \multicolumn{1}{c|}{$37.92\pm1.20$}  & $\textbf{39.83}\pm\textbf{0.79}$ \\
                            &                            & $a = 1.0$                                                                               & \multicolumn{1}{c|}{$37.17\pm0.53$} & \multicolumn{1}{c|}{$37.40\pm0.40$}  & $\textbf{40.69}\pm\textbf{0.55}$ \\
                            &                            & $IID$                                                                                   & \multicolumn{1}{c|}{$37.41\pm0.41$} & \multicolumn{1}{c|}{$37.22\pm0.23$} &  $\textbf{40.92}\pm\textbf{0.36}$ \\ \cline{2-6} 
                            & \multirow{4}{*}{mobileNet} & $a = 0.1$                                                                               & \multicolumn{1}{c|}{$30.94\pm0.32$} & \multicolumn{1}{c|}{$27.74\pm0.44$} & $\textbf{31.34}\pm\textbf{0.58}$ \\
                            &                            & $a = 0.5$                                                                               & \multicolumn{1}{c|}{$35.09\pm0.09$} & \multicolumn{1}{c|}{$31.72\pm0.19$}  &  $\textbf{37.41}\pm\textbf{0.12}$   \\
                            &                            & $a = 1.0$                                                                               & \multicolumn{1}{c|}{$34.75\pm0.11$} & \multicolumn{1}{c|}{$32.29\pm0.12$}  & $\textbf{37.95}\pm\textbf{0.10}$  \\
                            &                            & $IID$                                                                                   & \multicolumn{1}{c|}{$35.41\pm0.08$} & \multicolumn{1}{c|}{$33.68\pm0.06$}  &  $\textbf{38.06}\pm\textbf{0.08}$ \\ \hline
\multirow{3}{*}{FEMNIST}    & ResNet8                    & $-$                                                                                     & \multicolumn{1}{c|}{$80.22\pm0.07$} & \multicolumn{1}{c|}{$79.46\pm0.11$}  & $\textbf{82.01}\pm\textbf{0.13}$ \\
                            & VGG16                      & $-$                                                                                     & \multicolumn{1}{c|}{$82.99\pm0.16$} & \multicolumn{1}{c|}{$82.03\pm0.11$}  & $\textbf{84.28}\pm\textbf{0.22}$ \\
                            & mobileNet                  & $-$                                                                                     & \multicolumn{1}{c|}{$72.99\pm0.15$} & \multicolumn{1}{c|}{$73.37\pm0.14$}  & $\textbf{79.75}\pm\textbf{0.14}$ \\ \hline
\end{tabular}
}
\end{table*}

\subsection{Performance Comparison}
1) \emph{Comparison of Accuracy}: Table \ref{test acc} compares the test accuracy of our method and all baselines at IID and non-IID settings. In this table, the first column indicates the dataset type; The second column represents the model used for training; The third column denotes the data heterogeneity settings of clients, which specify different distributions for client data. The fourth column has three sub-columns that list the test accuracy information and its standard deviation for each of the three FL methods. From the table, we find that our method achieves the best accuracy in all cases. Noting that SFL is actually equivalent to FedAvg, so they have similar accuracy. Compared to SFL, the test accuracy of our approach is up to 16.5\% higher on the CIFAR-100 dataset with the VGG16 model when $a= 0.5$. We also find that $S^2FL$ was able to achieve higher accuracy on VGG16 than Resnet8 and mobileNet. This is because VGG16 has a more complex network structure compared to Resnet8 and mobileNet.All FL methods achieve a good performance on VGG16 but SFL still achieves the highest inference accuracy compared to other baselines.
Figure \ref{test_acc} shows the training process of $S^2FL$ and two baselines on CIFAR-10 based on the VGG16 model. We can find that our approach achieves the highest accuracy and the fastest convergence and has a more stable learning curve in both no-IID and IID scenarios.

2) \emph{Comparison of Training Convergence Speed and Communication Overhead}: We consider the time and communication overhead required to achieve a certain accuracy during training. Table \ref{train time and communication overhead} compares our method with two baselines, all of them training VGG16 on the CIFAR-10 dataset. As can be seen from the table, when training the large model (such as VGG16), SFL reduces the training time of the device by letting the device train a model portion, and because the output feature size of the model portion is smaller than the parameters of the entire model, $SFL$ has a faster convergence speed and smaller communication overhead than $FL$. Our approach mitigates the problem of straggles in $SFL$ and reduces the training time and communication overhead, $S^2FL$ outperforms $SFL$ by 3.54X and 2.57X in terms of training time and communication overhead when $a = 0.5$.

\begin{table*}[h]
\centering
\caption{Comparison of training time and communication overhead}
\label{train time and communication overhead}
\scalebox{1.0}{
\begin{tabular}{|c|c|cc|cc|cc|}
\hline
\multirow{2}{*}{\begin{tabular}[c]{@{}c@{}}Hetero.   \\      Settings\end{tabular}} & \multirow{2}{*}{Acc.\%} & \multicolumn{2}{c|}{FedAvg}          & \multicolumn{2}{c|}{SFL}             & \multicolumn{2}{c|}{$S^2FL$}         \\ \cline{3-8} 
                                                                                          &                            & \multicolumn{1}{c|}{Time}   & Comm.  & \multicolumn{1}{c|}{Time}   & Comm.  & \multicolumn{1}{c|}{Time}   & Comm.  \\ \hline\hline
\multirow{2}{*}{$a = 0.1$}                                                                & 40                         & \multicolumn{1}{c|}{21299s} & 29936M & \multicolumn{1}{c|}{13763s} & 12224M & \multicolumn{1}{c|}{\textbf{8471s}}  & \textbf{11372M} \\
                                                                                          & 50                         & \multicolumn{1}{c|}{41436s} & 61083M & \multicolumn{1}{c|}{35294s} & 32999M & \multicolumn{1}{c|}{\textbf{19471s}} & \textbf{27599M} \\ \hline
\multirow{2}{*}{$a = 0.5$}                                                                & 50                         & \multicolumn{1}{c|}{8162s}  & 19655M & \multicolumn{1}{c|}{7612s}  & 11958M & \multicolumn{1}{c|}{\textbf{2338s}}  & \textbf{5142M}  \\
                                                                                          & 60                         & \multicolumn{1}{c|}{21320s} & 51104M & \multicolumn{1}{c|}{16022s} & 25300M & \multicolumn{1}{c|}{\textbf{4524s}}  & \textbf{9826M}  \\ \hline
\multirow{2}{*}{$a = 1.0$}                                                                & 55                         & \multicolumn{1}{c|}{10498s} & 24191M & \multicolumn{1}{c|}{9329s}  & 14378M & \multicolumn{1}{c|}{\textbf{2908s}}  & \textbf{6381M}  \\
                                                                                          & 60                         & \multicolumn{1}{c|}{15162s} & 34472M & \multicolumn{1}{c|}{10058s} & 15423M & \multicolumn{1}{c|}{\textbf{3921s}}  & \textbf{8512M}  \\ \hline
\multirow{2}{*}{$IID$}                                                                    & 60                         & \multicolumn{1}{c|}{10614s} & 26913M & \multicolumn{1}{c|}{8963s}  & 15423M & \multicolumn{1}{c|}{\textbf{2952s}}  & \textbf{6823M}  \\
                                                                                          & 65                         & \multicolumn{1}{c|}{16020s} & 40520M & \multicolumn{1}{c|}{13422s} & 23047M & \multicolumn{1}{c|}{\textbf{3967s}}  & \textbf{9151M}  \\ \hline
\end{tabular}
}
\end{table*}

\subsection{Impacts of Different Configurations}
To demonstrate the generalizability and scalability of our approach in various scenarios, we examined the impact of different configurations on $S^2FL$ in the following three ways: different numbers of participating devices, different combinations of devices, and different size of client set.

1)\emph{Impacts of the number of participating devices}. We explored the impact of the number of devices $x$ participating in each round of training. We conducted experiments on four cases where $x$ is equal to 5, 10, 15, and 20. Figure \ref{impact1} shows the training process conducted on CIFAR-10 with IID distribution using VGG16. We observed that our approach achieved the highest accuracy across all cases, and as the number of devices increased, our approach had a higher accuracy improvement than other baselines. Furthermore, our method can achieve faster convergence when the number of devices involved in each round of training increases. This is because the more devices involved in training, the higher the probability that low FLOPS and low transfer rate devices will be selected in each round, which causes the straggler problem in FL and SFL. Our approach solves this problem by distributing the model adaptively for each device.

\begin{figure}[h]
    \centering
    \subfloat[$x=5$]{
    \includegraphics[width=0.23\textwidth]{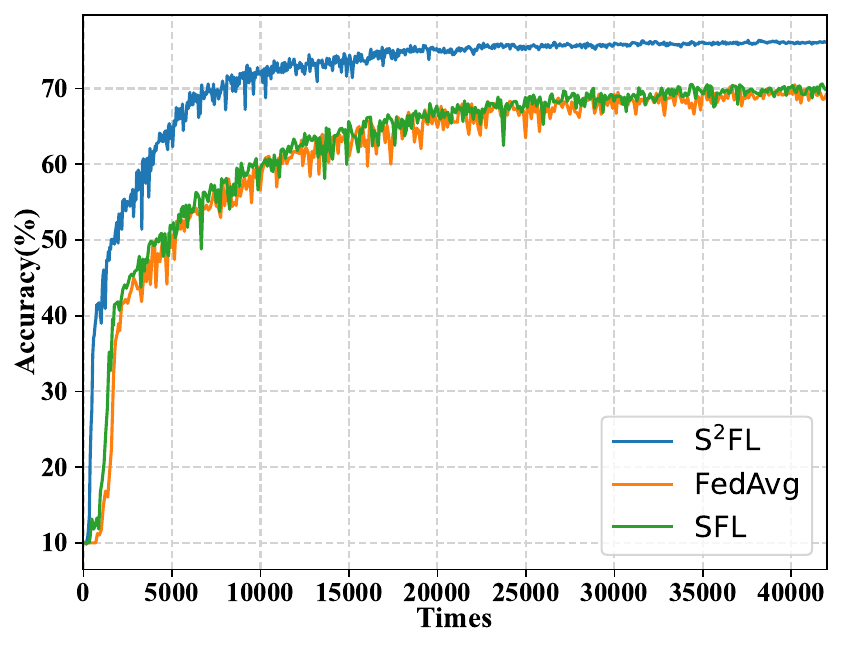}
        \label{fig:impact1_subfig1}
    }
    \subfloat[$x=10$]{
    \includegraphics[width=0.23\textwidth]{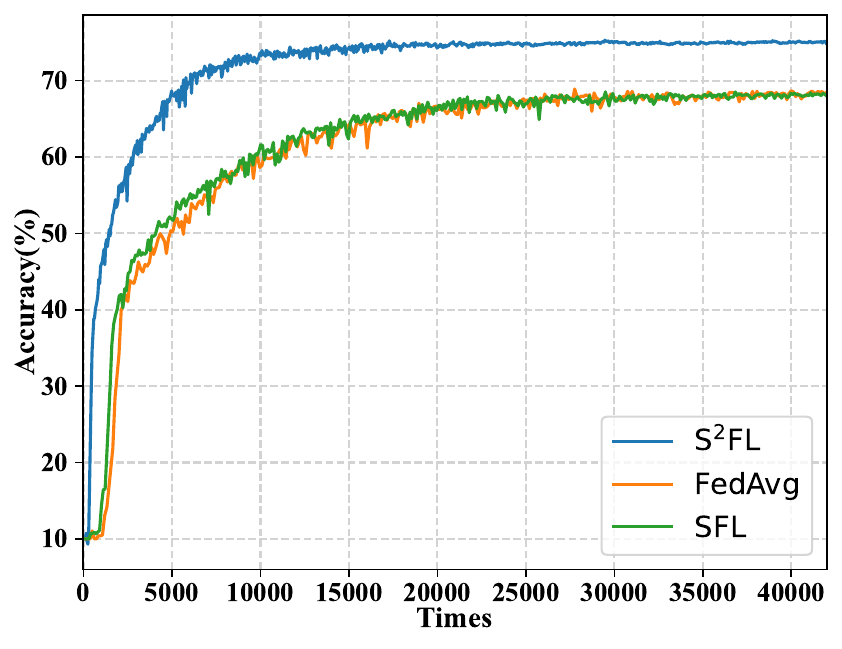}
        \label{fig:impact1_subfig2}
    }
    \hfill
    \subfloat[$x=15$]{
    \includegraphics[width=0.23\textwidth]{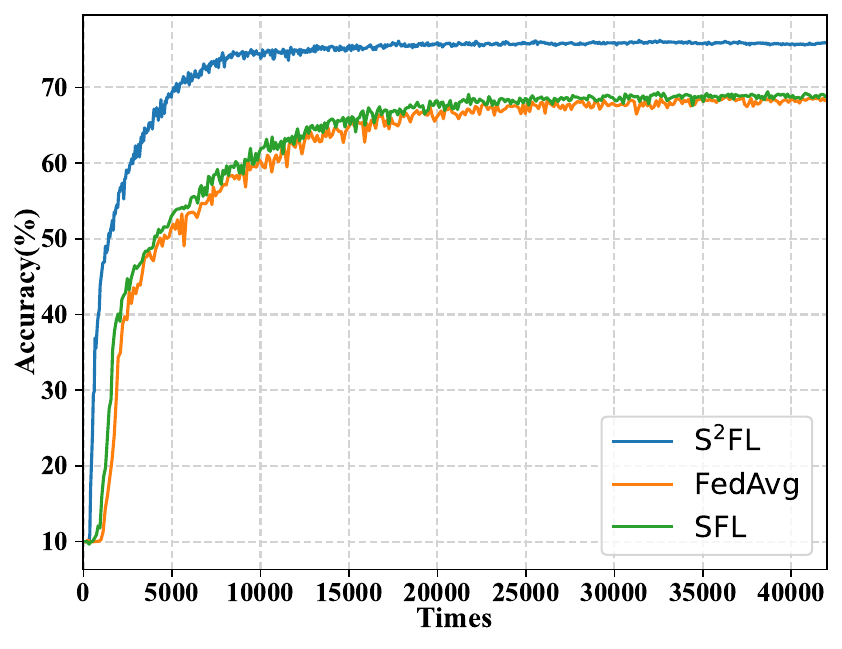}
        \label{fig:impact1_subfig3}
    }
    \subfloat[$x=20$]{
    \includegraphics[width=0.23\textwidth]{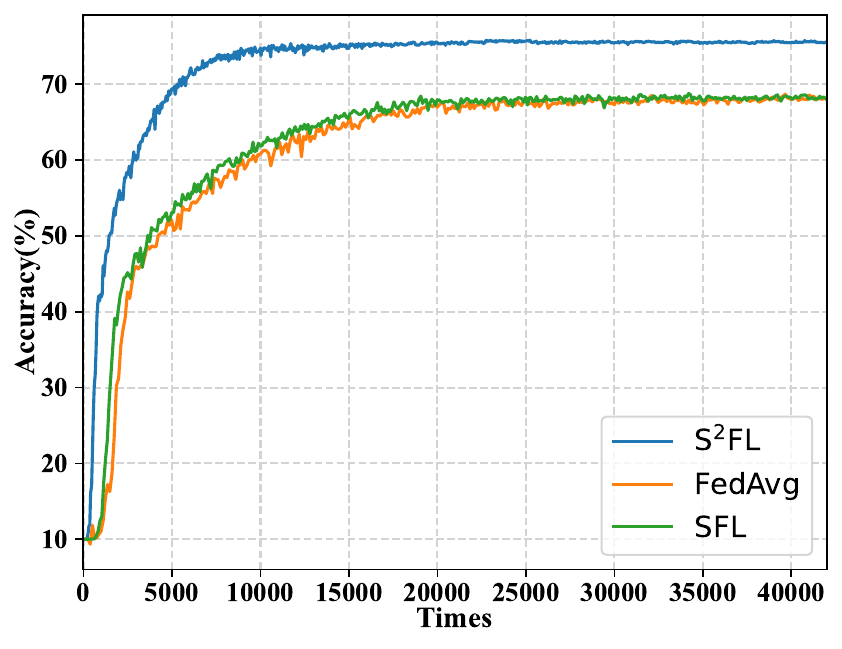}
        \label{fig:impact1_subfig4}
    }
    \caption{Training processes with different numbers of devices}
    \label{impact1} 
\end{figure}

\begin{figure}[h]
    \centering
    \subfloat[$conf 1$]{
    \includegraphics[width=0.23\textwidth]{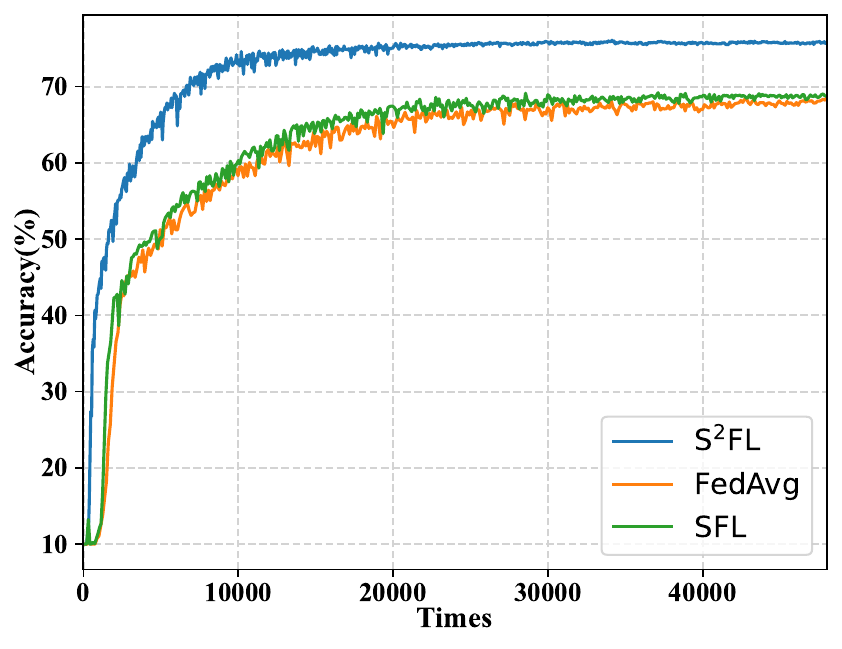}
        \label{fig:impact2_subfig1}
    }
    \subfloat[$conf 2$]{
    \includegraphics[width=0.23\textwidth]{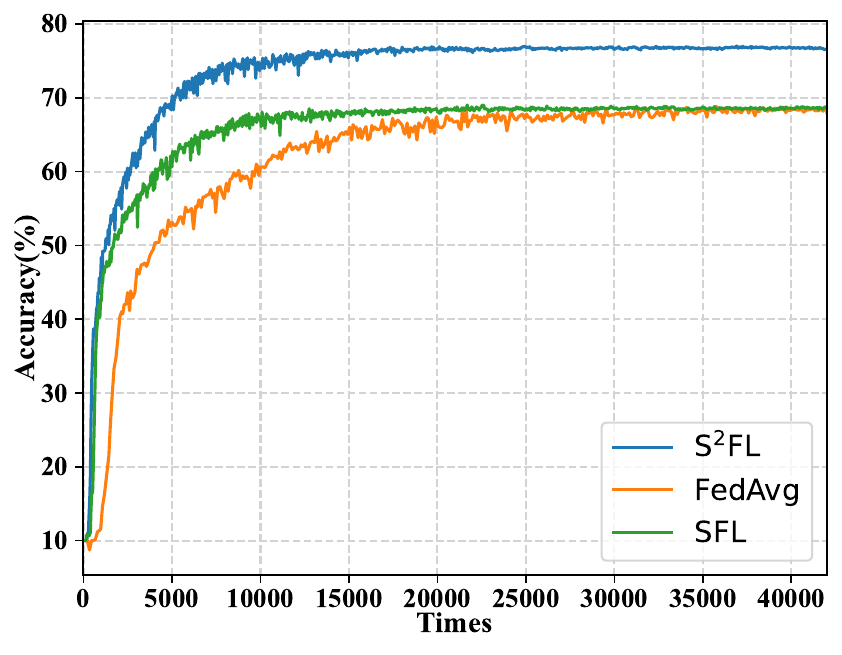}
        \label{fig:impact2_subfig2}
    }
    \caption{Training processes with different device compositions}
    \label{impact2}
\end{figure}

2) \emph{Impacts of Device Compositions}. We consider the effectiveness of our approach in different device composition scenarios. We used two different configurations in our experiment: \emph{conf 1}),  High: Mid: Low = 5: 3: 2 in both FLOPS settings and transfer rate settings. \emph{conf 2}), High: Mid: Low = 2: 3: 5 in both FLOPS settings and transfer rate settings. Figure \ref{impact2} shows the training process conducted on CIFAR-10 with IID distribution using VGG16. We observed that when the proportion of low FLOPS and low transfer rate devices increased, $S^2FL$ and $SFL$ achieved faster convergence than $FL$. This is because, for low FLOPS and low transfer rate devices, it is very time-consuming to train and transfer the complete model independently; $S^2FL$ and $SFL$ make these devices responsible for only a small portion of the model, thus speeding up convergence. It can be seen the $S^2FL$ achieves the highest accuracy and fastest convergence in both configurations. 

3) \emph{Impacts of the Size of Client Set}. We explore the effectiveness of our approach in different client set size while keeping the sampling rate constant (constant at 0.1) and the device composition. We conducted experiments on four cases where the client set size $|C|$ is equal to 20, 50, 100, 150. Figure 7 shows the training process conducted on CIFAR-10 with no-IID distribution($a= 0.5$) using VGG16. We observe that as the size of the client set increases, the improvement in accuracy and convergence speed of our approach compared to the baseline will also expand. This is because when keeping $a=0.5$, a larger set of clients means that individual clients have less data, making the difference in data distribution between clients larger. Thus, the accuracy of the two baselines of FedAvg and SFL decreases with the increase of $|C|$, while our method mitigates the accuracy degradation due to data heterogeneity through the \emph{Data Balance-based Training Mechanism}. Our method can achieve greater advantages in scenarios with more client devices and more heterogeneous data. Meanwhile, since the device composition and client sampling rate are unchanged, more low-performance devices are selected to participate in training when the client set becomes larger. This leads to slower convergence of $FL$ and $SFL$, while $S^2FL$ is almost unaffected.

\begin{figure}[h]
    \centering
    \subfloat[$|C|=20$]{
    \includegraphics[width=0.23\textwidth]{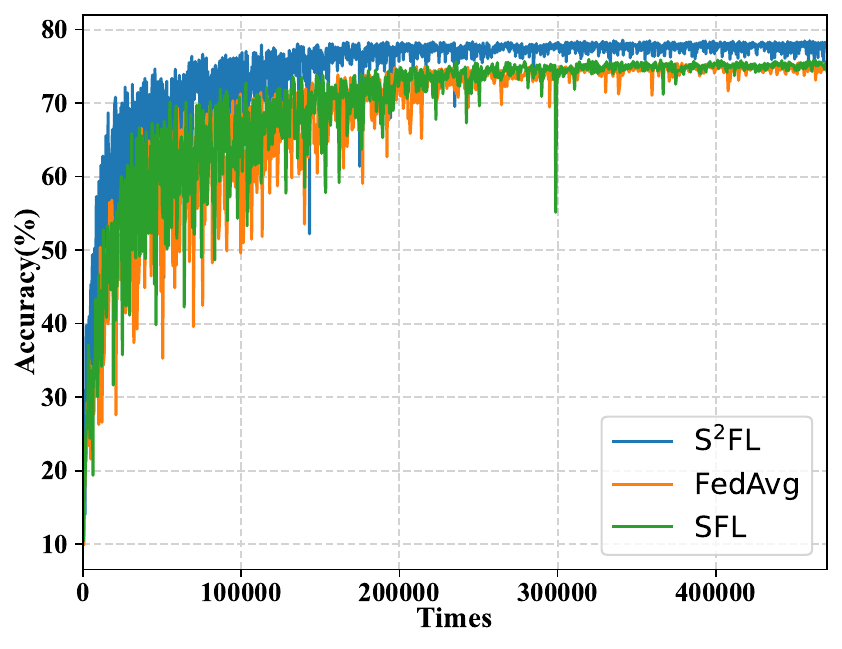}
        \label{fig:impact3_subfig1}
    }
    \subfloat[$|C|=50$]{
    \includegraphics[width=0.23\textwidth]{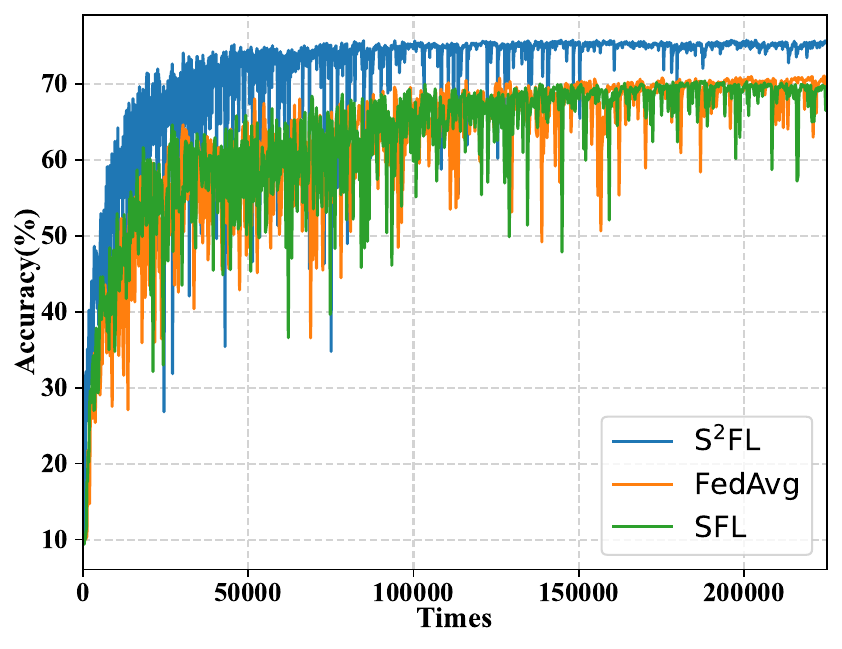}
        \label{fig:impact3_subfig2}
    }
    \hfill
    \subfloat[$|C|=100$]{
    \includegraphics[width=0.23\textwidth]{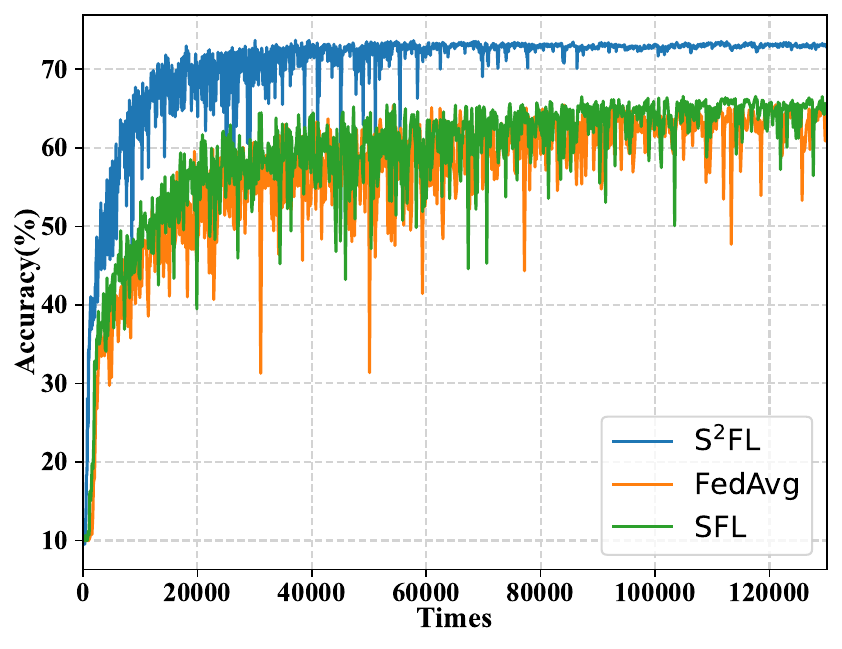}
        \label{fig:impact3_subfig3}
    }
    \subfloat[$|C|=150$]{
    \includegraphics[width=0.23\textwidth]{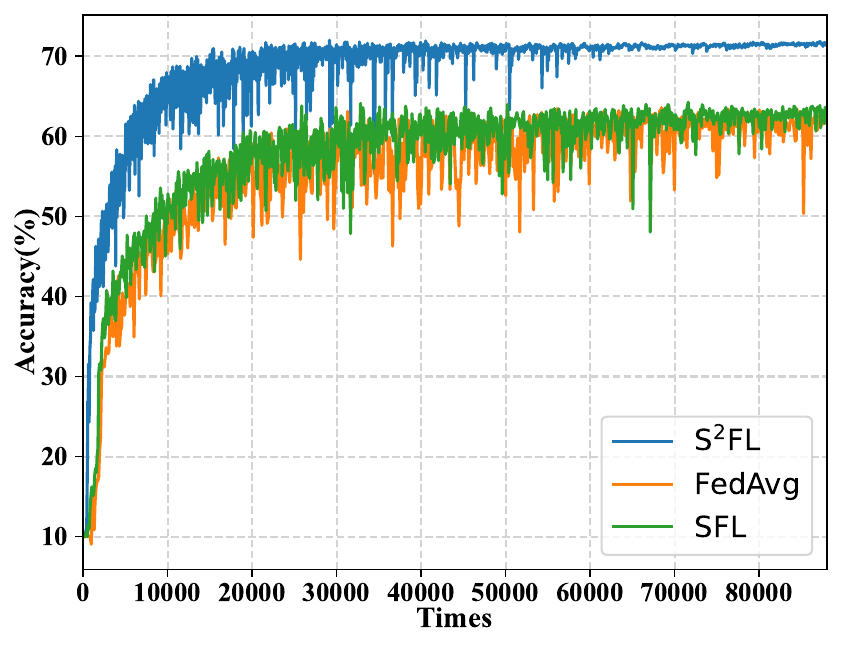}
        \label{fig:impact3_subfig4}
    }
    \label{impact3}
    \caption{Training processes with different size of client set}
\end{figure}

\subsection{Ablation Study}
To demonstrate the effectiveness of our adaptive sliding model split strategy and the data balance-based training mechanism, we conducted experiments on four different configurations of $S^2FL$: 1) $S^2FL + R$, which does not use the adaptive sliding model split strategy and the data balance-based training mechanism in $S^2FL$, note that this configuration is equivalent to SFL. 2) $S^2FL + B$ that uses the data balance-based training mechanism. 3) $S^2FL + M$ that uses the adaptive sliding model split strategy. 3) $S^2FL + MB$ that enables both the adaptive sliding model split strategy and the data balance-based training mechanism. Figure \ref{ablation} shows the ablation study results on the CIFAR-10 dataset with VGG16 following IID distribution. We can observe that $S^2FL + MB$ achieves the highest accuracy rate across all configurations. $S^2FL + M$ has a faster convergence speed than $S^2FL + R$, which indicates that our adaptive sliding model split strategy can effectively synchronize the training time between devices and speed up the training speed. $S^2FL + B$ has a higher test accuracy than $S^2FL + R$, which indicates that our data balance-based training mechanism can alleviate the problem of inference accuracy degradation caused by data heterogeneity.

\begin{figure}[h]
    \centering
    \includegraphics[width=0.95\linewidth]{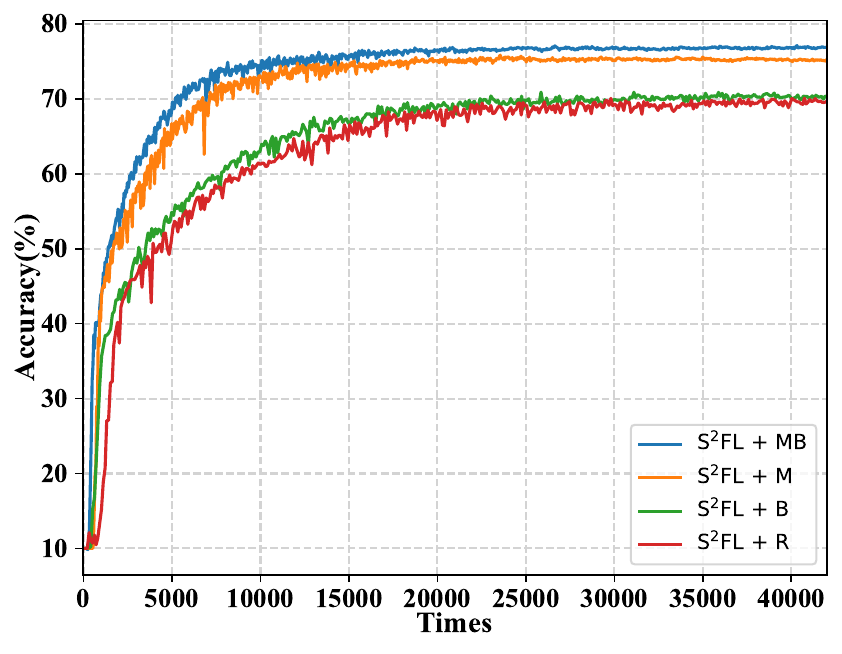}
    \caption{Ablation study in $S^2FL$}
    \label{ablation}
\end{figure} 


\section{Conclusion}
Although FL is increasingly used in AIoT scenarios, it still has the problem that resource-constrained devices cannot support large-size models. Some existing solutions such as SFL do not consider the problem of stragglers and data heterogeneity. To address this problem, we propose an effective and accurate approach called S$^2$FL based on SFL. It uses our proposed adaptive sliding model split strategy to select the model portion for the device that matches its computing resources so that the training time is similar between the devices. At the same time, S$^2$FL uses the data balance-based training mechanism to group devices, letting the model be trained on a more uniform data distribution. Comprehensive experiments show that our approach can achieve faster convergence speed and higher inference accuracy compared with state-of-the-art methods considering various uncertainties.






\bibliographystyle{ACM-Reference-Format}
\bibliography{sample-base}

\end{document}